\documentclass[twoside]{article}

\usepackage[utf8]{inputenc}
\usepackage[T1]{fontenc}
\usepackage{url}
\usepackage{booktabs}
\usepackage{amsfonts}
\usepackage{nicefrac}
\usepackage{microtype}
\usepackage{amsmath}
\usepackage{amsthm}
\usepackage{accents}
\usepackage{mathtools}
\usepackage{graphicx}
\usepackage{subcaption}
\usepackage[dvipsnames]{xcolor}
\usepackage{tcolorbox}
\usepackage{multirow}
\usepackage{cancel}
\usepackage{wrapfig}
\usepackage{footmisc}

\usepackage{listings}
\usepackage[round]{natbib}

\usepackage{hyperref}
\definecolor{tumblue}{RGB}{0, 101, 189}
\hypersetup{
   colorlinks=true,
   citecolor=tumblue,
   linkcolor=tumblue,
   filecolor=tumblue,
   urlcolor=tumblue,
 }
 \usepackage[capitalize]{cleveref}
\usepackage{algorithm}
\usepackage{algpseudocode}
\algrenewcommand\algorithmicrequire{\textbf{Input:}}
\algrenewcommand\algorithmicensure{\textbf{Output:}}
\algnewcommand\Parameters{\item[\textbf{Parameters:}]}

\usepackage{amsmath,amsfonts,amssymb,bm,dsfont}

\newtheorem{theorem}{Theorem}

\tcbset{
    theoremStyle/.style={
        colback=gray!5,
        colframe=gray!30,
        arc=0pt,
        boxsep=4pt,
        left=5pt, right=5pt, top=5pt, bottom=5pt,
        fonttitle=\bfseries,
        boxrule=0.4pt
    }
}

\tcolorboxenvironment{definition}{theoremStyle}
\tcolorboxenvironment{theorem}{theoremStyle}
\tcolorboxenvironment{lemma}{theoremStyle}

\def\vzero{{\bm{0}}}
\def\vone{{\bm{1}}}

\def\vtheta{{\bm{\theta}}}

\def\vb{{\bm{b}}}

\def\vf{{\bm{f}}}

\def\vh{{\bm{h}}}

\def\vu{{\bm{u}}}

\def\vw{{\bm{w}}}
\def\vx{{\bm{x}}}
\def\vy{{\bm{y}}}
\def\vz{{\bm{z}}}

\def\mW{{\bm{W}}}

\DeclareMathAlphabet{\mathsfit}{\encodingdefault}{\sfdefault}{m}{sl}
\SetMathAlphabet{\mathsfit}{bold}{\encodingdefault}{\sfdefault}{bx}{n}

\def\sA{{\mathcal{A}}}
\def\sB{{\mathcal{B}}}
\def\sC{{\mathcal{C}}}
\def\sD{{\mathcal{D}}}
\def\sE{{\mathcal{E}}}

\def\sN{{\mathcal{N}}}

\def\sV{{\mathcal{V}}}

\def\1{{\mathds{1}}}

\newcommand{\R}{\mathbb{R}}

\DeclareMathOperator*{\argmin}{arg\,min}

\newcommand{\norm}[1]{\left\Vert #1 \right\Vert}

\newcommand{\graph}{\operatorname{Graph}}
\newcommand{\proj}{\mathrm{P}}

\newcommand{\reflect}{\mathrm{R}}
\newcommand{\prox}{\operatorname{prox}}

\usepackage[accepted]{aistats2026}

\begin{document}

\twocolumn[

\aistatstitle{A Projection-Based Framework for Gradient-Free and Parallel Learning}

\aistatsauthor{
  Andreas Bergmeister
  \And Manish Krishan Lal
  \And Stefanie Jegelka
  \And Suvrit Sra
}

\aistatsaddress{
  TU Munich, MCML
  \And TU Munich, MCML
  \And TU Munich, MCML\\MIT CSAIL
  \And TU Munich, MCML\\MIT LIDS
} ]

\begin{abstract}
We present a feasibility-seeking approach to neural network training. This mathematical optimization framework is distinct from conventional gradient-based loss minimization and uses projection operators and iterative projection algorithms. We reformulate training as a large-scale feasibility problem: finding network parameters and states that satisfy local constraints derived from its elementary operations. Training then involves projecting onto these constraints, a local operation that can be parallelized across the network. We introduce PJAX, a JAX-based software framework that enables this paradigm. PJAX composes projection operators for elementary operations, automatically deriving the solution operators for the feasibility problems (akin to autodiff for derivatives). It inherently supports GPU/TPU acceleration, provides a familiar NumPy-like API, and is extensible. We train diverse architectures (MLPs, CNNs, RNNs) on standard benchmarks using PJAX, demonstrating its functionality and generality. Our results show that this approach is a compelling alternative to gradient-based training, with clear advantages in parallelism and the ability to handle non-differentiable operations.
\end{abstract}

\section{INTRODUCTION}
\label{sec:introduction}
\begin{figure*}[!t]
    \centering
    \includegraphics[width=0.9\textwidth]{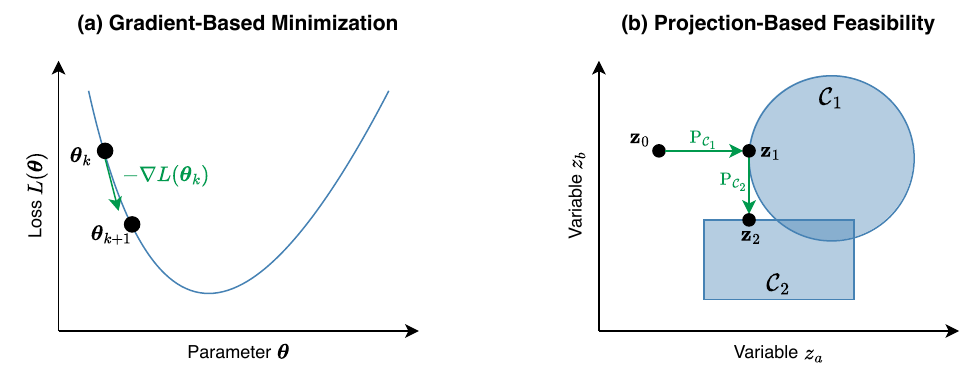}
    \caption{Neural network training paradigm shift. (a) Gradient-based methods iteratively minimize a loss function $L(\vtheta)$ using local gradients. (b) Our projection-based feasibility approach finds a point $\vz$ in the intersection of constraint sets (e.g., $\sC_1, \sC_2$) via iterative projections onto these sets.}
    \label{fig:paradigm_shift}
\end{figure*}

Deep learning models have achieved remarkable success across diverse applications, largely driven by the effectiveness of gradient-based optimization. The backpropagation algorithm~\citep{rumelhart1986learning}, paired with stochastic gradient descent (SGD) and its adaptive variants~\citep{duchi2011adaptive, tieleman2012lecture, kingma2014adam}, forms the bedrock of neural network training by efficiently computing loss gradients to iteratively adjust network parameters. Despite their undeniable success, gradient-based methods have several limitations. They often converge to local minima or high-error saddles~\citep{dauphin2014identifying, choromanska2015loss}, can suffer from vanishing or exploding gradients in deep architectures~\citep{hochreiter1991untersuchungen, bengio1994learning}, and fundamentally require network components and loss functions to be (sub)differentiable. Moreover, the sequential nature of backpropagation limits parallelism and prolongs update latency. Finally, the global error backpropagation mechanism, requiring symmetric feedback pathways, is widely considered biologically implausible~\citep{crick1989recent, lillicrap2016random}. These challenges, along with inspiration from neuroscience and alternative optimization paradigms, motivate the search for fundamentally different approaches to training neural networks.

The research landscape includes various alternatives to end-to-end backpropagation. Zeroth-order methods such as Evolution Strategies \citep{salimans2017evolution} or Genetic Algorithms \citep{holland1992adaptation} optimize the global loss using only function evaluations, at the cost of high sample complexity and poor scalability. Biologically inspired approaches use local rules but face distinct challenges: Hebbian methods \citep{Hebb1949} lack supervised error integration, while spiking neural networks \citep{song2000competitive, gerstner2002spiking} contend with non-differentiable dynamics and credit assignment difficulties. Gradient approximation methods like Target Propagation \citep{lee2015difference}, Feedback Alignment \citep{lillicrap2016random} attempt to alleviate some backpropagation issues but often trade off convergence speed and performance. The recent \emph{Forward–Forward} algorithm~\citep{hinton2022forward} eliminates the backward pass altogether by training each layer on a local contrastive ``goodness'' objective; though its effectiveness on large-scale benchmarks has yet to be demonstrated.

This paper investigates a paradigm shift: reformulating neural network training from a loss minimization task (\cref{fig:paradigm_shift}a) into a large-scale \emph{feasibility problem} (\cref{fig:paradigm_shift}b). Instead of navigating a complex loss landscape, we seek to identify network parameters and intermediate states that simultaneously satisfy a collection of local constraints derived from the network's structure and the desired input-output mappings. Our approach achieves this through a fine-grained decomposition of the network into its elementary operations—termed \emph{primitive functions} (e.g., inner products, pointwise activations). For these primitive functions, orthogonal projections onto their graphs (set of valid input–output pairs) are often computationally inexpensive. The availability of such efficient projections allows us to recast training as the problem of finding a point in the intersection of numerous local constraint sets, a task well-suited for iterative projection algorithms rooted in convex optimization~\citep{bauschke2011convex}. This feasibility perspective builds upon conceptual work by~\citet{elser2021learning}.

This feasibility-driven paradigm offers several inherent advantages over traditional gradient-based methods. First, training relies on the availability of projection operators for primitive functions rather than on their differentiability, naturally accommodating non-differentiable components within network architectures. Second, the updates are local, modifying only adjacent variables in the computation graph. This eliminates the need for global error backpropagation and its associated weight transport problem, aligning more closely with notions of biological plausibility. Crucially, this locality, particularly when coupled with a bipartite structuring of the computation graph (\cref{sec:method}), enables fully parallelizable updates across the network. While other strategies also decompose the global training objective into simpler, coupled sub-problems to enable parallel execution, such as layer-wise ADMM~\citep{glowinski1975approximation, boyd2011distributed} or the Method of Auxiliary Coordinates (MAC)~\citep{carreira2014distributed}, they often face computational challenges such as large matrix inversions or intricate dual-variable bookkeeping~\citep{taylor2016training}. In contrast, our fine-grained decomposition requires only a collection of efficiently computable local projections.

This paper delivers the rigorous formulation, a robust software framework, and systematic empirical validation necessary to establish projection-based training as a concrete, implementable, and explorable alternative. Our specific contributions are:
\begin{enumerate}
\setlength{\itemsep}{1pt}
    \item A detailed graph-structured \emph{feasibility formulation} using edge variables. This formulation allows us to solve the problem efficiently with parallelizable projection algorithms (e.g., Alternating Projections and Douglas-Rachford) by exploiting a bipartitioning of the computation graph (\cref{sec:method}).
    \item A set of projection operators for key primitive functions (e.g., dot product, ReLU of sum, max pooling) that are fundamental building blocks for neural networks. The mathematical derivations for these operators are provided in \cref{sec:primitive_projections}.
    \item \textbf{PJAX} (Projection JAX)\footnote{PJAX is available at \url{https://github.com/AndreasBergmeister/pjax}}, a complete numerical framework built upon JAX~\citep{jax2018}. Designed for this compositional projection-based paradigm, it serves a role analogous to automatic differentiation systems for gradient-based methods. PJAX inherits JAX's GPU/TPU acceleration and JIT compilation capabilities, provides a familiar API mirroring NumPy/JAX, automatically orchestrates the iterative solution of user-defined feasibility problems, and supports extension with new primitive functions and projection operators.
    \item Extensive empirical validation across diverse neural network architectures (MLPs, CNNs, RNNs) on standard benchmarks (\cref{sec:experiments}). These experiments demonstrate the viability of our projection-based approach and provide an initial characterization of its performance.
\end{enumerate}

\section{BACKGROUND AND PRELIMINARIES}
\label{sec:background}
We work in finite-dimensional Euclidean spaces $\R^d$. The inner product is $\langle \vx, \vy \rangle = \vx^\top \vy$, inducing the Euclidean norm $\|\vx\| = \sqrt{\langle \vx, \vx \rangle}$, which in a product space is $\|(\vx_1, \ldots, \vx_n)\|^2 = \sum_{i=1}^n \|\vx_i\|^2$.

\subsection{Projection and Proximal Operators}
\label{sec:proj-prox}
The concept of projection onto sets is central to our method. Given a non-empty closed set $\sC \subseteq \R^d$, the \emph{projection} of $\vx \in \R^d$ onto $\sC$ is
\begin{equation}
    \label{eq:proj}
    \proj_{\sC}(\vx) = \argmin_{\vy \in \sC} \|\vx - \vy\|^2.
\end{equation}
This projection always exists. If $\sC$ is convex, $\proj_{\sC}(\vx)$ is unique; the operator $\proj_{\sC}$ is non-expansive, and its fixed points constitute $\sC$. If $\sC$ is non-convex, the minimizer may not be unique, rendering $\proj_{\sC}$ set-valued; $\proj_{\sC}(\vx)$ then denotes an arbitrary choice from the set of minimizers.

Projections onto product sets $\sC = \sC_1 \times \dots \times \sC_m$ (where each $\sC_i \subseteq \R^{d_i}$ is non-empty and closed) separate as follows: for $(\vx_1, \ldots, \vx_m) \in \R^{d_1} \times \dots \times \R^{d_m}$,
\begin{equation}
    \label{eq:separable_projections}
    \proj_{\sC_1 \times \dots \times \sC_m}(\vx_1, \ldots, \vx_m) = (\proj_{\sC_1}(\vx_1), \ldots, \proj_{\sC_m}(\vx_m)).
\end{equation}

The \emph{proximal operator} is associated with functions. For a proper, lower semi-continuous function $f: \R^d \to (-\infty, +\infty]$ and $\lambda > 0$, the proximal operator of $f$ at $\vx_0 \in \R^d$ is
\begin{equation}
    \label{eq:prox}
    \prox_{\lambda f}(\vx_0) = \argmin_{\vx \in \R^d} \left( f(\vx) + \frac{1}{2\lambda} \|\vx - \vx_0\|^2 \right).
\end{equation}
If $f$ is convex, this minimizer is unique~\citep{moreau1965proximite}. A point $\vx$ is a minimizer of $f$ if and only if it is a fixed point of the proximal operator, $\prox_{\lambda f}(\vx) = \vx$. The proximal operator generalizes the projection operator: if $f$ is the indicator function of a closed convex set $\sC$, then $\prox_{\lambda f}(\vx_0) = \proj_{\sC}(\vx_0)$.

\subsection{Feasibility Problems and Projection Algorithms}
\label{sec:feasibility}
Many problems involve finding a point in the intersection of multiple constraint sets. Given closed sets $\sC_1, \ldots, \sC_N \subseteq \R^d$, the \emph{feasibility problem} seeks $\vx \in \R^d$ such that
\begin{equation}
    \label{eq:multiset_feasibility}
    \vx \in \bigcap_{i=1}^N \sC_i,
\end{equation}
assuming a non-empty intersection. Iterative projection algorithms are apt for such problems, particularly when individual projections $\proj_{\sC_i}$ are computationally simpler than directly finding a point in the intersection.

Classical algorithms include \textbf{Alternating Projections (AP)} for two sets $\sC_1, \sC_2$, with the sequence:
\begin{equation} \label{eq:ap_iter}
    \vx_{k+1} = \proj_{\sC_1}(\proj_{\sC_2}(\vx_k)).
\end{equation}
\textbf{Cyclic Projections (CP)} extends this to $N > 2$ sets:
\begin{equation} \label{eq:cp_iter}
    \vx_{k+1} = \proj_{\sC_N}(\proj_{\sC_{N-1}}(\dots \proj_{\sC_1}(\vx_k))).
\end{equation}
\textbf{Douglas-Rachford (DR)} for two sets $\sC_1, \sC_2$ uses reflections $\reflect_{\sC}(\vx) = 2 \proj_{\sC}(\vx) - \vx$:
\begin{equation} \label{eq:dr_iter}
    \vx_{k+1} = \tfrac{1}{2} \left( \vx_k + \reflect_{\sC_1}(\reflect_{\sC_2}(\vx_k)) \right).
\end{equation}
When the sets $\sC_i$ are convex with a non-empty intersection, AP~\citep{bregman1965method} and CP~\citep{gubin1967method} converge to a point in this intersection. For DR, under similar conditions, the sequence of projections (e.g., $\{\proj_{\sC_2}(\vx_k)\}$) converges to such a point~\citep{douglas1956numerical, lions1979splitting}. These algorithms form the basis for solving feasibility problems in our work.

\section{METHOD}
\label{sec:method}

\begin{figure*}[!t]
    \centering
    \includegraphics[width=0.9\textwidth]{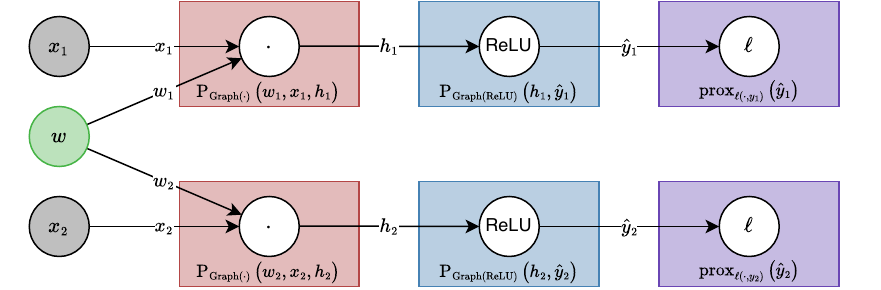}
    \caption{Computation graph for $\ell(\operatorname{ReLU}(w \cdot x_i), y_i)$ on two samples, showing projection operators for hidden function and loss nodes.}
    \label{fig:dag}
\end{figure*}

Consider a supervised learning setting with dataset $\sD = \{(\vx_i, \vy_i)\}_{i=1}^N$, where $\vx_i \in \R^{d_\text{in}}$ and $\vy_i \in \R^{d_\text{target}}$. Let $\vf: \R^{d_\text{in}} \times \R^{d_\theta} \to \R^{d_\text{out}}$ be a parametric function (e.g., a neural network), and let $\ell: \R^{d_\text{out}} \times \R^{d_\text{target}} \to \R$ be a loss function. The conventional approach minimizes the empirical risk
\begin{equation} \label{eq:erm}
    \min_{\vtheta \in \R^{d_\theta}} \frac{1}{N} \sum_{i=1}^N \ell\bigl(\vf(\vx_i, \vtheta), \vy_i\bigr).
\end{equation}
We reformulate training as a \emph{feasibility problem} by translating the per-sample objective of minimizing the loss into a set of hard constraints. The architecture of $\vf$, viewed as a composition of elementary operations, provides further constraints: each operation's output must align with its inputs according to its defining function. Training then becomes the search for a state (network parameters and internal activations) that simultaneously satisfies this entire collection of local constraints. We formally construct this feasibility problem by defining these constraints over variables within a computation graph.

\subsection{Constraint formulation via computation graph}
\label{sec:computation_graph}

For a given input $\vx$, parameters $\vtheta$, and target $\vy$, we represent the computation of $\ell(\vf(\vx, \vtheta), \vy)$ as a directed acyclic graph (DAG) $G = (\sV, \sE)$, the \emph{computation graph}. Nodes represent constant inputs, parameters, elementary operations of $\vf$ (primitive scalar functions), and the loss function. Edges represent data flow: an edge $(u, v) \in \sE$ indicates that the value from node $u$ serves as an input to the operation at node $v$. We denote the set of parent nodes of $v$ as $\sN^-(v)$ and its children as $\sN^+(v)$.
Crucially, we decompose $\vf$ into its scalar component functions (e.g., addition, inner products, ReLU activations); the nodes in the graph that correspond to operations therefore represent these scalar functions. Consequently, only scalar values pass along the edges. This granularity allows us to derive tractable projection operators onto the constraint sets associated with these primitive functions.

We introduce a variable $z_{uv} \in \R$ for each edge $(u,v) \in \sE$, representing the value carried along that edge. The state vector $\vz = (z_{uv})_{(u,v)\in\sE}\in\R^{\sE}$ collects all edge variables. For each node in $\sV$, we define a constraint set in $\R^{\sE}$. For a single sample $(\vx, \vy)$ (where $\vx=(x_1, \dots, x_{d_{\text{in}}})$ and $\vy=(y_1, \dots, y_{d_{\text{target}}})$), we categorize these constraints by node type as follows:

\textbf{Constant input nodes} $c_{j} \in \sV$ correspond to each input component $x_j$ for $j \in [d_\text{in}]$. The constraints
\begin{equation*}
    \sC_{c_j} = \{ \vz \in \R^{\sE} \mid \forall w \in \sN^+(c_j): z_{c_j w} = x_j \}
\end{equation*}
enforce that edges outgoing from the respective node carry its value. The projection operator $\proj_{\sC_{c_j}}(\vz)$ sets $z_{c_j w}$ to $x_j$.

\textbf{Parameter nodes} $p_{k} \in \sV$ correspond to each parameter component $\theta_k$ for $k \in [d_\theta]$. The constraints
\begin{equation*}
    \sC_{p_k} = \{ \vz \in \R^{\sE} \mid \forall u, w \in \sN^+(p_k): z_{p_k w} = z_{p_k u} \}
\end{equation*}
enforce consensus among all outgoing edge values from a parameter node. The projection operator $\proj_{\sC_{p_k}}(\vz)$ sets the values of outgoing edges $z_{p_k w}$ to the average of the current values $\{ z_{p_k w} \mid w \in \sN^+(p_k) \}$.

\textbf{Target node} $t \in \sV$ receives the network's output and ensures minimal loss with respect to the target $\vy$. Depending on the loss function $\ell$, we either project onto the constraint set
\begin{equation*}
    \sC_{t} = \{ \vz \in \R^{\sE} \mid (z_{u t})_{u \in \sN^-(t)} \in \argmin_{\vy'} \ell(\vy', \vy) \}
\end{equation*}
or apply the proximal operator of the loss function (with $\lambda > 0$) to the current predicted outputs
\begin{equation*}
    (z'_{u t})_{u \in \sN^-(t)} \leftarrow \prox_{\lambda \ell(\cdot, \vy)}\bigl((z_{u t})_{u \in \sN^-(t)}\bigr).
\end{equation*}
See \cref{sec:output_operators} for details. 
Note that a fixed point of the proximal operator is a minimizer of the loss function (see \cref{sec:proj-prox}), so it satisfies the constraint $\sC_t$.

\textbf{Hidden function nodes} $h \in \sV$ represent the application of a primitive function $f_h$ to their inputs. The constraints
\begin{equation*}
    \begin{aligned}
    \sC_h = \{ \vz \in \R^{\sE} \mid
    &\ \forall w \in \sN^+(h): \\
    &\ z_{h w} = f_h \bigl( (z_{u h})_{u \in \sN^-(h)} \bigr) \}
    \end{aligned}
\end{equation*}
enforce that all outgoing edge values equal the result of applying $f_h$ to its inputs. Projecting onto $\sC_h$ involves:
(1) computing an average of current values on outgoing edges, $\bar{z} = (\sum_{w \in \sN^+(h)} z_{h w}) / |\sN^+(h)|$;
(2) projecting the incoming edge values $(z_{u' h})_{u' \in \sN^-(h)}$ and the average $\bar{z}$ onto the graph of $f_h$: $(z'_{u' h}, z'_{h w}) = \proj_{\graph(f_h)}((z_{u' h})_{u' \in \sN^-(h)}, \bar{z})$;
(3) set incoming edge values $z_{u' h} \leftarrow z'_{u' h}$ and outgoing edge values $z_{h w} \leftarrow z'_{h w}$ for all $w \in \sN^+(h)$.
\Cref{thm:consensus-constrained-projection} formally justifies these steps.

The overall feasibility problem involves finding a state vector $\vz$ that lies in the intersection of all such individual node constraints
\begin{equation} \label{eq:feasibility_intersection}
    \text{Find } \vz \in \bigcap_{v \in \sV} \sC_v.
\end{equation}

The formulation extends to batches of samples by constructing a single, larger computation graph with $N$ instances of data-dependent components (input nodes, function nodes, and loss nodes for each sample $i$), while sharing parameter nodes across instances. Conceptually, we construct a computation graph for the function 
\begin{align}
    \label{eq:batch_loss}
    & ((\vx_1, \vy_1), \ldots, (\vx_N, \vy_N), \vtheta) \mapsto \nonumber\\ &\quad (\ell(\vf(\vx_1, \vtheta), \vy_1), \ldots, \ell(\vf(\vx_N, \vtheta), \vy_N)),
\end{align}
as \cref{fig:dag} illustrates.
The following theorem justifies the correctness of our formulation, showing that if a solution to the feasibility problem exists, it yields a set of parameters that minimizes the empirical risk in \cref{eq:erm}.

\begin{theorem}[Optimality of feasibility solution]
    \label{thm:optimality_feasibility}
    Let $\vz^* \in \R^{\sE}$ be a solution to the feasibility problem in \cref{eq:feasibility_intersection} for the function $\ell (\vf(\vx, \vtheta), \vy)$ applied to a dataset $\sD$ as described above.
    Then the consensus values of the outgoing edges from the parameter nodes $p_k$ $(k \in [d_\theta])$ in $\vz^*$ minimize the empirical risk in \cref{eq:erm}.
\end{theorem}
\begin{proof}
    A feasible state vector $\vz^*$ satisfies all node constraints. Specifically, satisfying the target node constraints means that the outputs of the network, $\hat{\vy}_i^*$, minimize the per-sample loss. Parameter node constraints ensure their outgoing edges yield common consensus values, $\theta_k^*$, which define the overall network parameters $\vtheta^* \in \R^{d_\theta}$. The input and hidden function node constraints then ensure that these loss-minimizing values $\hat{\vy}_i^*$ are precisely the outputs of the network function $\vf(\vx_i, \vtheta^*)$. Consequently, $\vtheta^*$ minimizes the empirical risk in \cref{eq:erm}.
\end{proof}

We solve the feasibility problem in \cref{eq:feasibility_intersection} with the iterative projection algorithms described in \cref{sec:feasibility}. When the intersection is non-empty and the sets $\sC_v$ are convex, these algorithms converge to a feasible point under standard assumptions. In the nonconvex setting, they should be viewed as heuristic iterative methods that often find useful approximate solutions in practice, but without comparable general convergence guarantees. To efficiently compute the projections in parallel across the computation graph, we leverage the structure of the graph, specifically its bipartite nature, as the following theorem details. Notably, neural network computation graphs are often bipartite; if not, they can be made bipartite by inserting identity operations (dummy nodes with $f_v(z) = z$).

\begin{theorem}[Parallelizable projections via bipartition]
    \label{thm:bipartite_projection}
    Let $G = (\sV, \sE)$ be a computation graph with a bipartition $\sV = \sA \cup \sB$, where $\sA$ and $\sB$ are disjoint sets of nodes. The constraints for each node $v \in \sV$ are defined as above. Then, the feasibility problem \cref{eq:feasibility_intersection} is equivalent to the two-set feasibility problem
    \begin{align}
        \text{Find} & \quad \vz \in \sC_{\sA} \cap \sC_{\sB}, \\
        \text{where} & \quad \sC_{\sA} = \bigcap_{v \in \sA} \sC_v \text{ and }\sC_{\sB} = \bigcap_{v \in \sB} \sC_v.
    \end{align}
    Furthermore, the projection $\proj_{\sC_{\sA}}(\vz)$ (and analogously $\proj_{\sC_{\sB}}(\vz)$) can be computed by independently (and in parallel) applying the projection operators $\proj_{\sC_v}(\vz)$ for all $v \in \sA$ (or $v \in \sB$).
\end{theorem}
\begin{proof}
    The equivalence $\bigcap_{v \in \sV} \sC_v = \sC_{\sA} \cap \sC_{\sB}$ is definitional.
    Consider distinct nodes $u, w \in \sA$. By the bipartition, they are not adjacent. Since $\proj_{\sC_u}$ only modifies edge variables incident to $u$ (similarly for $w$), and $u,w$ share no incident edges, these projections act on disjoint sets of coordinates in $\vz$.
    Therefore, all operators $\{\proj_{\sC_v}\}_{v \in \sA}$ modify mutually disjoint components of $\vz$. By \cref{eq:separable_projections} (projections onto product sets), $\proj_{\sC_{\sA}}(\vz)$ is then computed by applying these individual projections independently, enabling parallel execution. An analogous argument for $\proj_{\sC_{\sB}}(\vz)$ reduces the original problem to a two-set feasibility problem with parallelizable projection steps.
\end{proof}

\Cref{alg:training} summarizes the complete projection-based training procedure with batch processing.

\begin{algorithm}[!t]
    \caption{Projection-based training}
    \label{alg:training}
    \begin{algorithmic}[1]
        \Require Model $\vf$, Loss function $\ell$, Initial parameters $\vtheta_0$, Dataset $\sD$, Batch size $B$, Projection steps per batch $K$, Projection method \texttt{ProjMethod}.
        \Ensure Optimized parameters $\vtheta$.
        \State $\vtheta \leftarrow \vtheta_0$
        \While{not converged}
            \State $((\vx_1, \vy_1), \ldots, (\vx_B, \vy_B)) \sim \sD$
            \State $G \gets \text{computation graph for batch loss (\ref{eq:batch_loss})}$
            \State $\vz \leftarrow \text{initial edge state vector} \in \R^{\sE}$
            \For{$k = 1$ to $K$}
                \State $\vz \leftarrow \texttt{ProjMethod}(G, \vz, \{\proj_{\sC_v}\}_{v \in \sV})$
            \EndFor
            \State $\vtheta \leftarrow \text{extract parameters from } \vz$
        \EndWhile
        \State \Return $\vtheta$
    \end{algorithmic}
\end{algorithm}

\section{COMPLEXITY ANALYSIS}
\label{sec:complexity}
Having detailed the projection-based training methodology, we now analyze its computational and memory complexity. 
All considered projection algorithms perform a single projection per step onto the constraints $\sC_v$ for each node $v \in \sV$. We focus on the complexity of these projections, as they dominate the overall computational cost.

\paragraph{Computational complexity per step}
Projecting onto a single node's constraint $\sC_v$ typically has a cost comparable to the forward evaluation of the primitive function $f_v$. For example, the projection for the dot product, which implements neuron pre-activation, involves solving a scalar quintic equation (e.g., via 5 Newton steps) followed by rescaling inputs/outputs. Similarly, the projection for ReLU computes two candidate solutions in closed form and selects the one minimizing distance. See \cref{sec:primitive_projections} for details on these and other primitive functions.
Therefore, the total cost per step of the projection method is roughly proportional to the cost of executing all primitive operations in the computation graph, similar to a forward pass. However, unlike the sequential forward and backward passes that backpropagation requires, the projection updates within each partition ($\sA$ or $\sB$) are fully parallelizable.

\paragraph{Memory complexity}
Storing the state vector $\vz \in \R^{\sE}$ leads to $O(|\sE|)$ memory complexity. In contrast, standard backpropagation only requires storing a primitive function's output and its gradient, resulting in $O(|\sV|)$ complexity. Since many networks have significantly more edges $|\sE|$ than nodes $|\sV|$, our method generally requires more memory. This requirement is pronounced in architectures with extensive weight sharing. For instance, during \textbf{batch processing}, processing $N$ samples requires replicating the state associated with shared parameters $N$ times in $\vz$, as each sample interacts with the parameters via distinct edges in the expanded graph. Similarly, in \textbf{sequence models (e.g., RNNs)}, unrolling over $T$ time steps necessitates storing distinct edge states for shared parameters at each step. \textbf{Convolutional networks (CNNs)} also exhibit this effect: applying a convolutional kernel at multiple spatial locations means each application corresponds to graph connections that require separate edge states. Our experiments made this evident: the high memory demands of the 4-layer CNN necessitated reducing the number of hidden units per layer to 16 to fit within the available 96GB GPU memory. In contrast, architectures with less parameter sharing, like MLPs, could accommodate larger sizes. Gradient-based methods are more memory-efficient in these cases because they aggregate gradients for shared parameters (e.g., by summation) and only need to store one copy of the parameters and their accumulated gradients.

In summary, the projection-based approach trades potentially higher memory usage, especially with shared parameters, for significant gains in parallelizability compared to gradient-based methods. The overall convergence rate (number of steps) depends on the specific problem and projection algorithm used. However, as our experiments empirically demonstrate (\cref{sec:experiments}), projection-based training can achieve convergence rates that are competitive with first-order SGD methods on several tasks.

\section{EXPERIMENTS}
\label{sec:experiments}
This section presents an empirical evaluation of our projection-based training method, outlined in \cref{alg:training}. We compare three projection methods (\texttt{ProjMethod}), Alternating Projections (AP), Douglas-Rachford (DR), and Cyclic Projections (CP), against standard gradient-based optimizers, Stochastic Gradient Descent (SGD) and Adam~\citep{kingma2014adam}, and non-backpropagation baselines, Feedback Alignment (FA)~\citep{lillicrap2016random} on MLPs and Forward–Forward (FF)~\citep{hinton2022forward} on MLPs and CNNs. For AP and DR, we utilize the two-set feasibility problem derived from the bipartite graph formulation (\cref{thm:bipartite_projection}). For CP, we apply the multi-set feasibility problem formulation directly, projecting cyclically onto node constraints and merging constraints within the same layer for efficiency (e.g., performing all ReLU projections in parallel). The projection order for CP follows a backward breadth-first search (BFS) from the target node.

\subsection{Datasets and tasks}
We evaluate our approach on diverse standard machine learning tasks. These include image classification using \textbf{MNIST} ($28 \times 28$ grayscale images, 10 classes) and \textbf{CIFAR-10} ($32 \times 32$ color images, 10 classes), binary classification on the \textbf{HIGGS} dataset (distinguishing signal from background noise in particle physics), and character-level language modeling with the \textbf{Shakespeare} dataset. We use the default train/test splits (approximately 80\% train, 20\% test) and reserve 10\% of training data for validation (hyperparameter tuning and early stopping). Across all tasks, we employ the standard cross-entropy loss and report accuracy as the primary evaluation metric. For the sequence modeling task, accuracy specifically refers to the fraction of correctly predicted next characters.

\subsection{Neural network architectures}
We test three representative architectures:
A \textbf{Multilayer Perceptron (MLP)} with fully connected layers (linear + bias) and ReLU activations.
A \textbf{Convolutional Neural Network (CNN)} with $3 \times 3$ convolutional layers (stride 1, padding 1) and ReLU activations; the output feature maps are spatially max-pooled into vectors before passing through a final linear layer.
We implement a \textbf{Recurrent Neural Network (RNN)} with an MLP cell that processes the concatenation of a learned character embedding and the previous hidden state ($\vh_{t-1}$) at each step $t$. The MLP outputs both logits for the next character and a raw hidden state ($\vh_t'$). Activating this state via $\vh_t = \operatorname{ReLU}(\vh_t')$ produces the input state for step $t+1$.

For each architecture, we evaluate both shallow (1 hidden layer) and deep (4 hidden layers) variants. For deep models, we also test a skip-connection design tailored to projection-based training: we concatenate the post-activation outputs of all hidden layers and feed the resulting vector to the final linear readout. This shortens the path from early layers to the output, which is important because projection-based training relies on local projection operators that only update adjacent variables in the computation graph. Without additional shortcuts, information from early layers reaches the output only after many iterations. Crucially, the design preserves the bipartite structure of the computation graph and thus our parallel projection scheme. It differs from the usual ResNet-style residual connection, which replaces a module $f(x)$ with $x+f(x)$ and does not limit the distance to the output as effectively.

\paragraph{Implementation and hyperparameters}
We implement projection-based methods (AP, DR, CP) using our PJAX framework (\cref{sec:pjax}) and all baselines with JAX and Flax \citep{flax2020github}, ensuring fair comparison, as PJAX uses JAX as its backend. Notably, defining models in PJAX's \texttt{pjax.nn} API closely mirrors standard JAX/Flax usage. For the output layer constraint, we use the proximal operator for cross-entropy (\cref{thm:cross-entropy-prox}, with $\lambda=5$), finding it more robust than margin constraints (\cref{thm:margin-loss}), which require careful tuning of the margin parameter.

FA and FF are implemented in line with their original formulations \citep{lillicrap2016random, hinton2022forward}. Neither method is applied to skip-connection architectures, which are not part of their original design. For FF, we report results with both SGD and Adam optimizers, since performance varies substantially with the choice of optimizer.

We use consistent parameters across experiments: learning rate $10^{-3}$ for SGD, Adam and FF, $10^{-4}$ for FA, batch size $256$ for all methods, and $K=50$ projection steps per batch for AP/DR/CP. Each projection step counts as one training step, so projection methods process $50$ times fewer batches than baseline methods per reported step. Accordingly, the horizontal axes in \cref{fig:training_curves,fig:mnist_mlp_depth_skip} compare optimization steps rather than equal amounts of data processed.

\subsection{Results}
\label{sec:results}
\begin{figure*}[!t]
    \centering
    \begin{subfigure}[b]{0.3\textwidth}
        \includegraphics[width=\linewidth]{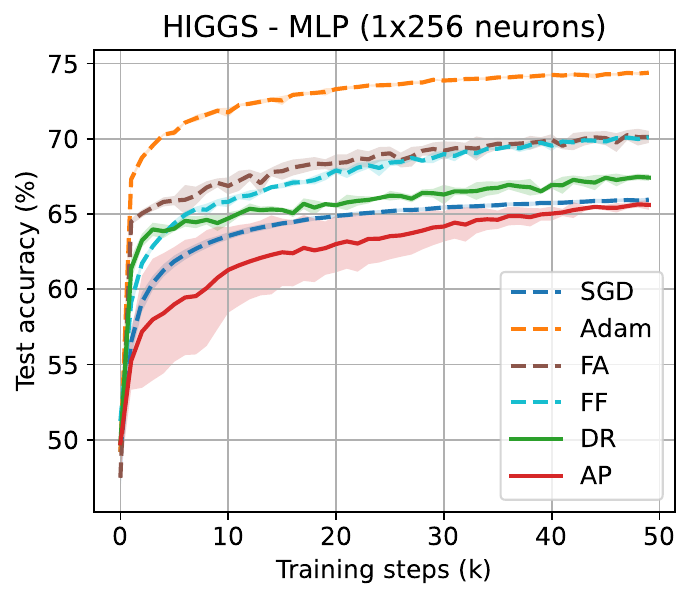}
    \end{subfigure}
    \hfill
    \begin{subfigure}[b]{0.3\textwidth}
        \includegraphics[width=\linewidth]{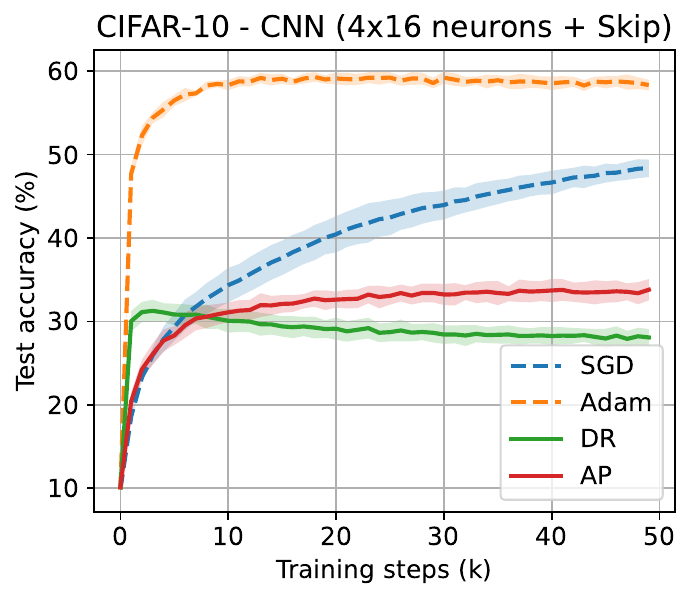}
    \end{subfigure}
    \hfill
    \begin{subfigure}[b]{0.3\textwidth}
        \includegraphics[width=\linewidth]{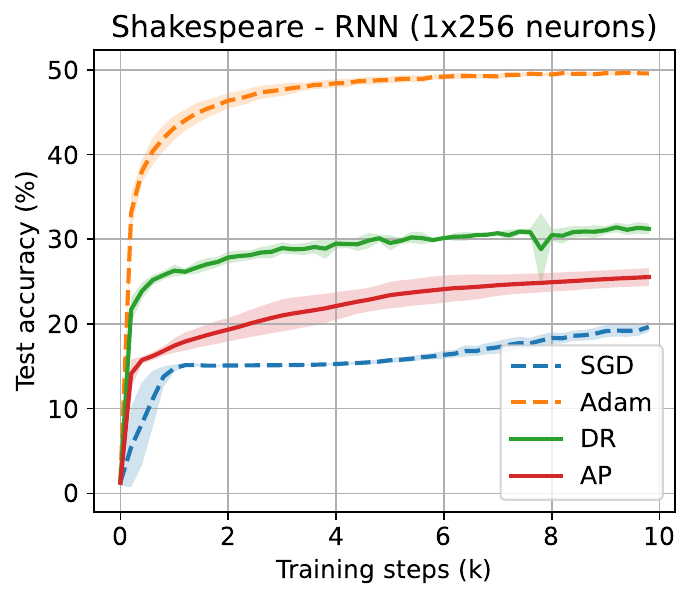}
    \end{subfigure}
    \caption{Test accuracy vs. training steps.}
    \label{fig:training_curves}
\end{figure*}

\begin{figure*}[!t]
    \centering
    \begin{subfigure}[b]{0.32\textwidth}
        \includegraphics[width=\linewidth]{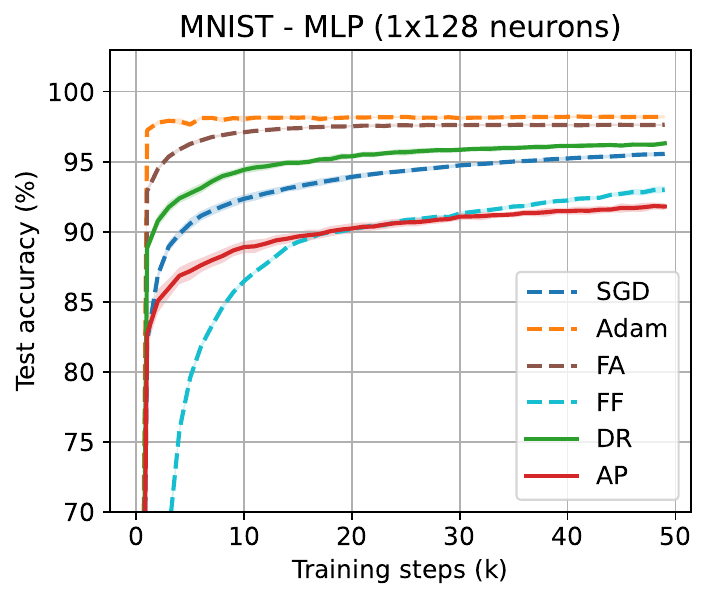}
    \end{subfigure}
    \hfill
    \begin{subfigure}[b]{0.32\textwidth}
        \includegraphics[width=\linewidth]{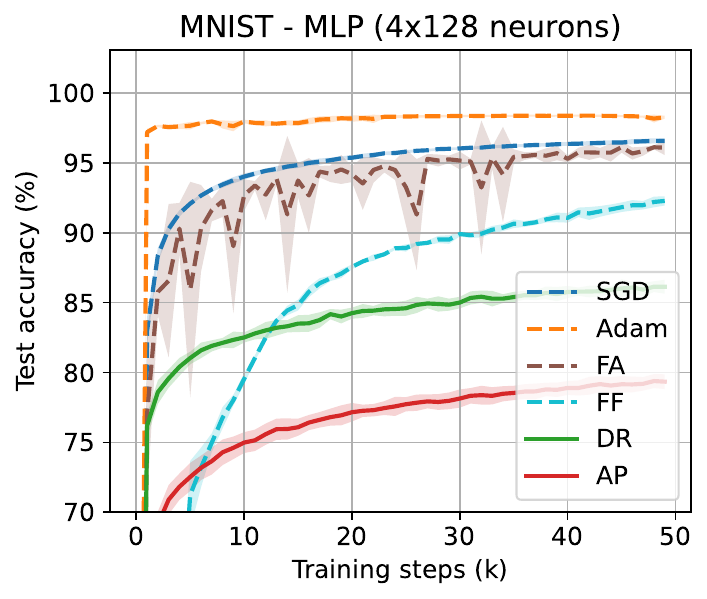}
    \end{subfigure}
    \hfill
    \begin{subfigure}[b]{0.32\textwidth}
        \includegraphics[width=\linewidth]{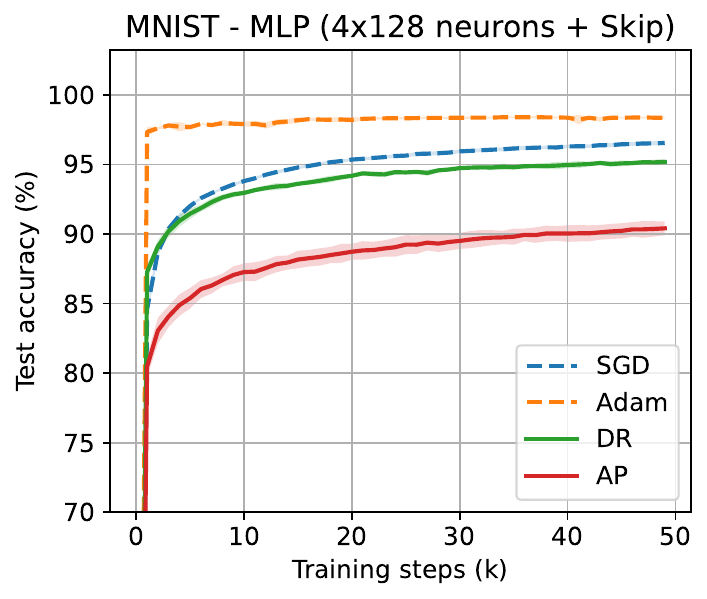}
    \end{subfigure}
    \caption{Impact of network depth and skip connections.}
    \label{fig:mnist_mlp_depth_skip}
\end{figure*}

We present an overview of our experimental findings here. \Cref{sec:results_tables} provides detailed numerical results, including final test accuracies, convergence steps, and timings of all conducted experiments on a single NVIDIA H100 GPU with 96GB of memory. 

\paragraph{Overall trends.}
\Cref{fig:training_curves} displays representative test accuracy curves from our experiments, highlighting a general trend: while Adam consistently achieves the highest test accuracy and generally converges fastest, projection-based methods prove viable across tasks and architectures. Among the projection algorithms, DR often outperforms AP on MLPs and RNNs in accuracy, while AP is competitive on CNNs. CP yields almost identical accuracy to AP but is slower due to its sequential nature (hence we omit CP from plots for clarity; see \cref{sec:results_tables} for results). We also find that AP/CP tend to exhibit higher run-to-run variance than DR, likely due to greater sensitivity to initialization. DR's reflection (overshoot) step enlarges the effective basin of attraction and reduces stalling in shallow basins, leading to more consistent outcomes across seeds. The specific characteristics of projection-based training, such as convergence to optimal solutions and step efficiency, vary with network architecture. For \textbf{MLPs} (\cref{tab:mlp_results}), projection methods, particularly DR in shallow cases, approach SGD's accuracy with notable computational efficiency per step (often $\sim 10\times$ faster). For \textbf{CNNs} (\cref{tab:cnn_results}), the accuracy gap relative to gradient-based methods widens, and step speed decreases due to higher memory requirements for shared parameters (see \cref{sec:complexity} for theoretical details). For \textbf{RNNs} (\cref{tab:rnn_results}), projection methods possess a structural advantage, as no backpropagation through time is performed, thereby sidestepping issues with vanishing or exploding gradients (see \cref{fig:rnn}). This results in significantly faster convergence in terms of training steps compared to SGD, though Adam still achieves the best overall results, managing training dynamics effectively. The memory demands for shared parameters in RNNs also impact step times, similar to CNNs.

\paragraph{FF and FA baselines.}
Our experiments do not reveal a consistent ranking between projection-based methods and FF/FA baselines for all tasks. On CIFAR-10 CNN, for example, the 4-layer projection-based model with skip connections outperforms both FF (even when trained with Adam) and FA (implemented without skips, as in prior work). In other settings, FF or FA achieve better performance. A clear pattern is that FF is highly optimizer-dependent: it performs well with Adam but degrades sharply with SGD. \Cref{fig:training_curves,fig:mnist_mlp_depth_skip} report FF results with Adam, while additional SGD results are provided in \cref{sec:results_tables}.

\paragraph{Depth and skip connections.}
A key consideration for projection-based methods, which rely on local updates, is their effectiveness in deeper networks. \Cref{fig:mnist_mlp_depth_skip} explores this using an MLP trained on MNIST with the Douglas-Rachford (DR) algorithm. While shallow MLPs (left panel) train readily, the performance of a deep MLP without skip connections (center panel) degrades, highlighting challenges in propagating information through many layers via local projections alone. The introduction of skip connections (right panel) substantially improves training for the deep MLP. These connections provide shorter paths for information flow, proving crucial for effective learning in deeper architectures using projection methods, while maintaining the bipartite graph structure essential for parallelization.

\begin{figure}
    \centering
    \includegraphics[width=0.9\columnwidth]{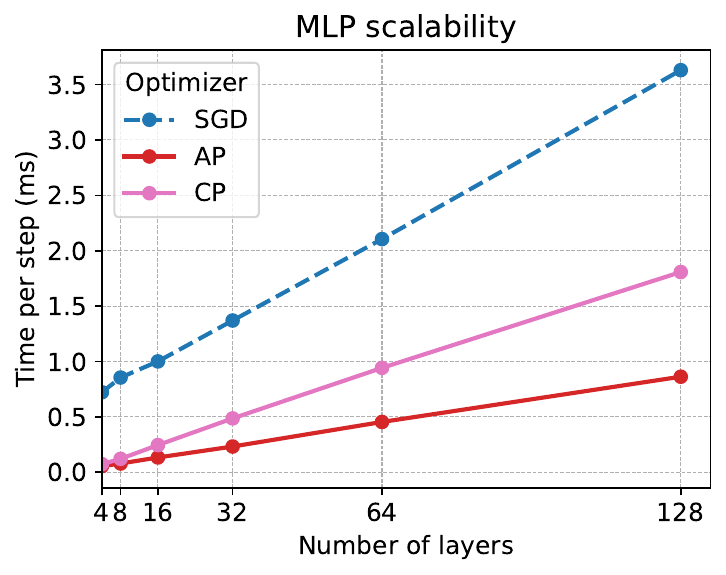}
    \caption{Scalability of projection vs. gradient-based methods with network depth.}
    \label{fig:mlp_scalability}
    \vspace{-0.8em}
\end{figure}

\begin{figure*}[!t]
    \centering
    \includegraphics[width=0.9\textwidth]{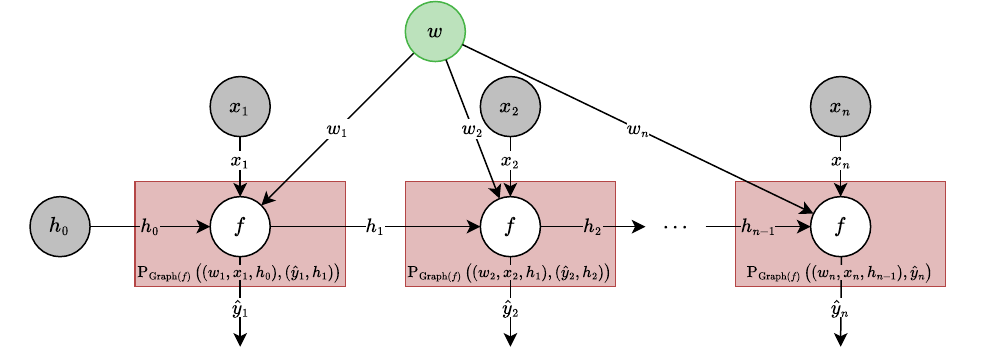}
    \caption{Computation graph for RNN: each unrolled cell $f$ has a local parameter copy ($w_t$). Projections are performed separately (possibly concurrently) for each time step. Parameter consensus is enforced through the parameter node.}
    \label{fig:rnn}
    \vspace{-0.8em}
\end{figure*}

\Cref{fig:mlp_scalability} further analyzes the computational efficiency of projection methods. It compares step times of SGD, CP, and AP for an MLP (16 hidden units per layer) on a single MNIST sample, across varying network depths. These algorithms serve as direct proxies for the gradient-based, sequential projection, and parallel projection paradigms, respectively, avoiding the complexities of other optimizers like Adam or more involved projection schemes like DR. The analysis reveals two main trends. First, both CP and AP achieve lower step times than SGD, likely due to their intensive data reuse ($K=50$ updates per sample). Second, AP's step time increases less steeply with network depth compared to SGD and CP. This favorable scaling stems from the parallelization of updates across all layers via the bipartite graph structure (\cref{thm:bipartite_projection}), a concurrency efficiently harnessed by modern hardware like GPUs, in contrast to the sequential processing inherent in SGD (backpropagation) and CP (cyclic projections). Consequently, AP's parallel advantage becomes more pronounced for deeper networks.

\section{DISCUSSION}
\label{sec:discussion}
We reframe neural network training as a feasibility problem and solve it with iterative projections. Because the method requires projections rather than derivatives, it naturally supports non-differentiable components such as quantization (\cref{thm:quantization}) and logical constraints (e.g., the margin loss in \cref{thm:margin-loss}). Updates are local to adjacent variables in the computation graph, which removes global error backpropagation, enables parallelization across network components, and aligns with biological learning principles. The PJAX framework (\cref{sec:pjax}) composes projection operators analogously to how autodiff composes derivatives via the chain rule, providing a practical tool for investigating and extending this paradigm. To the best of our knowledge, it is the only gradient-free training approach accompanied by a general-purpose framework rather than a single-task implementation.

Beyond these advantages, our formulation keeps the standard vector-to-vector interface with single-pass inference. In contrast, \emph{Forward--Forward}~\citep{hinton2022forward} learns a scalar ``goodness'' for each (input, label) pair, which implies a per-class forward pass at inference and presents challenges for scaling to large output spaces (as in language modeling) or extending to continuous targets.

Empirically, projection methods train diverse architectures (MLPs, CNNs, and RNNs) reliably. They achieve competitive step times and benefit from parallel updates; RNNs particularly benefit from avoiding backpropagation through time. Douglas–Rachford is often the most stable, though performance still falls short of highly optimized adaptive gradient methods such as Adam.
For both gradient- and projection-based optimizers, the train-test gap is often modest (\cref{sec:results_tables}), suggesting that the gap reflects limits in fitting the training data rather than a fundamental lack of generalization.

The main limitation is memory. The algorithm maintains distinct edge variables for each interaction between parameters and data, so requirements scale with batch size, sequence length, and the number of convolutional locations.

Promising directions include: (i) improved projection dynamics (adaptive damping/relaxation, preconditioning, acceleration, learned step sizes); (ii) hybrid schemes that interleave projections with occasional gradient steps; (iii) tailored architectures that shorten paths to the output and use structured sparsity (e.g., Mixture-of-Experts) to limit parameter–data interaction; (iv) memory-saving techniques (low-rank or quantized edge states); and (v) refined analyses of convergence and generalization for nonconvex, possibly set-valued projections. We hope PJAX lowers the barrier to exploring these directions.

\section*{ACKNOWLEDGMENTS}
This project was funded by the Alexander von Humboldt Foundation.

\bibliographystyle{plainnat}
\bibliography{references}
\clearpage
\appendix
\thispagestyle{empty}
\onecolumn
\aistatstitle{A Projection-Based Framework for Gradient-Free and Parallel Learning: \\
Supplementary Materials}

\section{RELATED WORK}
\label{sec:related_work} 
Deep neural network training is overwhelmingly dominated by gradient-based optimization methods. However, the limitations of gradient-based methods and inspiration from neuroscience have driven research into non-gradient-based learning. This section reviews prominent families of these methods.

\subsection{Gradient-based optimization: The standard paradigm}
Gradient-based methods are foundational in deep learning. The backpropagation algorithm, as popularized in the seminal work by \citet{rumelhart1986learning}, remains the dominant paradigm for training neural networks. Backpropagation leverages the chain rule of differentiation to compute the gradient of a global loss function with respect to all network parameters efficiently. These gradients are then typically used within variants of stochastic gradient descent (SGD) to iteratively update the parameters and minimize the loss.

To address issues such as slow convergence, navigation of complex loss landscapes, and sensitivity to learning rates, several improvements to vanilla SGD have been proposed. Momentum-based methods, including Polyak's momentum \citep{polyak1964some} and Nesterov accelerated gradient \citep{nesterov1983method}, incorporate a history of past updates to accelerate progress along consistent descent directions and dampen oscillations. Adaptive methods such as AdaGrad \citep{duchi2011adaptive}, RMSProp \citep{tieleman2012lecture}, and Adam \citep{kingma2014adam} adjust learning rates on a per-parameter basis, adapting to the geometry of the loss landscape and helping networks train more robustly, especially with sparse gradients.
For large models, memory-efficient variants of backpropagation have also been developed, e.g., reducing activation storage or recomputation overheads \citep{korthikanti2023reducing, xu2025activationsharding}.

Despite their success, gradient-based methods face several challenges. They can converge to suboptimal local minima or encounter prevalent high-error saddle points, particularly in high-dimensional non-convex optimization problems~\citep{dauphin2014identifying, choromanska2015loss}. Furthermore, they can suffer from vanishing or exploding gradients in deep networks \citep{hochreiter1991untersuchungen, bengio1994learning}, require the objective function and network components to be differentiable (or sub-differentiable), and can be sensitive to hyperparameter choices. The computational cost and sequential nature of computing and propagating gradients through deep networks can also be significant bottlenecks for parallelization across layers. These limitations motivate the exploration of alternative training schemes.

\subsection{Derivative-free and zeroth-order methods}
A direct alternative is to forgo gradient calculation entirely and use derivative-free optimization (DFO) or zeroth-order methods. These methods typically rely only on evaluations of the objective function (the global loss).
Evolution Strategies (ES) \citep{rechenberg1978evolutionsstrategien, schwefel1977evolutionsstrategien} are a prominent example. ES optimizes parameters by sampling points in the parameter space around the current estimate, evaluating the loss at these points, and moving the estimate in a direction informed by these evaluations, effectively approximating a gradient direction or directly seeking improvement \citep{salimans2017evolution}. While computationally intensive due to the need for multiple forward passes per update, ES has shown surprising effectiveness, particularly in reinforcement learning domains \citep{salimans2017evolution}.
Other population-based methods like Genetic Algorithms (GAs) \citep{holland1992adaptation} and Particle Swarm Optimization (PSO) \citep{kennedy1995particle} maintain a population of candidate solutions (network parameters) and iteratively refine them using operators inspired by biological evolution (selection, crossover, mutation) or social behavior (particle movement). The objective is typically to maximize a fitness function directly related to the global loss. Evolutionary approaches have also been applied to meta-learning problems, such as optimizing learning rules or hyperparameters \citep{schmidhuber1987evolutionary}.
The primary objective of these DFO methods is the minimization of the global loss function. However, they often exhibit high sample complexity compared to gradient-based methods, requiring many loss evaluations, which can be prohibitive for large networks and datasets. Scalability to extremely high-dimensional parameter spaces remains a challenge.

\subsection{Biologically inspired local learning rules}
Drawing inspiration from neuroscience, another class of methods employs local learning rules that do not require global error backpropagation.
Rooted in neuroscientific theory, Hebbian learning follows the principle ``cells that fire together, wire together.'' This principle, introduced by Donald Hebb in his 1949 book \citep{Hebb1949}, states that the synaptic strength between two neurons increases if they are co-activated. Early neural network research used simplified Hebbian rules for unsupervised representation learning. Modern variants, such as Oja's rule \citep{oja1982simplified} and the Bienenstock--Cooper--Munro (BCM) rule \citep{bienenstock1982theory}, build on the classical Hebbian framework to ensure weight stabilization and learn principal components or other low-dimensional representations. Relatedly, attractor networks like the Hopfield network \citep{hopfield1982neural} also utilize Hebbian-style rules to store patterns as stable states (attractors) within an energy function, suitable for tasks like associative memory and pattern completion.
Another important family of models leveraging local updates includes energy-based models like Boltzmann Machines (BMs) and particularly their simplified variant, Restricted Boltzmann Machines (RBMs) \citep{hinton2002training}. RBMs are stochastic networks trained to model a probability distribution over their inputs, typically via Maximum Likelihood. While exact gradient calculation is intractable for general BMs, RBMs can be effectively trained using algorithms like Contrastive Divergence (CD) \citep{hinton2002training}, which relies on Gibbs sampling and local computations between connected layers to approximate the gradient. This made RBMs instrumental as building blocks for Deep Belief Networks \citep{hinton2006fast}, offering a biologically more plausible route to unsupervised feature learning and generative modeling.
A common appealing property of these various local learning rules, whether Hebbian or energy-based, is their independence from global gradient signals. Instead, they rely on updates based on local neuronal activity (e.g., correlations, timing, sampling statistics) or principles like energy minimization. The objective is often related to maximizing correlation, capturing variance, achieving representational stability, minimizing a network energy function, or modeling data distributions, rather than directly minimizing a global supervised loss function for input-output mappings. This locality can be advantageous for large-scale or biologically plausible architectures where global signals are unavailable or costly. However, directly applying or adapting these principles to match the performance of gradient-based methods on complex supervised tasks remains challenging.

Inspired by the brain's event-driven architecture, neuromorphic computing focuses on spiking neural networks (SNNs) that communicate via discrete spike events instead of continuous-valued activations. Conventional backpropagation becomes difficult in spiking systems due to the non-differentiable nature of spike functions. As a result, a variety of alternative learning rules, such as Spike-Timing-Dependent Plasticity (STDP) in spiking neural networks \citep{song2000competitive, gerstner2002spiking}, have emerged. STDP aligns closely with Hebbian-like principles, adjusting synaptic strengths based on the relative timing of pre- and post-synaptic spikes. The objective here is local synaptic modification driven by temporal correlations. Neuromorphic chips like IBM's TrueNorth \citep{merolla2014million} and Intel's Loihi \citep{davies2018loihi} have been designed to implement such models efficiently, potentially offering significant energy savings. Because learning in SNNs often occurs locally within each synapse, neuromorphic approaches can circumvent the need for global gradient signals. However, training SNNs for complex, real-world tasks comparable to those solved by deep learning remains an active research area, and performance often lags behind conventional networks. Such methods also inherently handle non-differentiable spike events, a property shared with the projection-based methods discussed in this work.

\subsection{Gradient approximation and decoupled training methods}
Several approaches attempt to retain some benefits of gradient-based learning while avoiding full backpropagation, often motivated by biological plausibility (e.g., the weight transport problem) or computational considerations.
Target Propagation (TP) \citep{bengio2014auto, lee2015difference} aims to compute layer-specific targets instead of gradients. An inverse mapping or autoencoder associated with each layer generates target activations that, if achieved, would reduce the global loss. The layer then updates its weights locally to better map its input to these targets. The objective is local: minimize the mismatch between layer outputs and computed targets, serving as a proxy for the global loss. TP can avoid propagating precise gradients but requires learning or defining inverse mappings, which can be challenging and potentially unstable.
Equilibrium Propagation \citep{scellier2017equilibrium} is another line of work in this direction: it computes updates from differences between free and weakly clamped network states in energy-based models, avoiding explicit reverse-mode backpropagation through layers.
Feedback Alignment (FA) \citep{lillicrap2016random} replaces the transposed weight matrices used in backpropagation's backward pass with fixed, random feedback matrices. Surprisingly, the network can learn by aligning its forward weights to leverage these random pathways for error signaling. The objective remains the global loss, but the gradient information is approximated. FA solves the weight transport problem but typically learns slower and may achieve lower final performance compared to backpropagation. Direct Feedback Alignment (DFA) \citep{nokland2016direct} further simplifies this by sending the error signal directly from the output layer to each hidden layer via fixed random matrices.
Decoupled Neural Interfaces (synthetic gradients) \citep{jaderberg2017decoupled} similarly reduce strict sequential dependencies by learning gradient predictors for intermediate modules.
Layer-wise training provides another way to decouple learning. Initially popular for pre-training Deep Belief Networks \citep{hinton2006fast}, this involves training layers sequentially, often using unsupervised objectives like reconstruction error before potential end-to-end fine-tuning. The objective is local per layer/stage. While useful for initialization, purely greedy layer-wise training may not yield globally optimal solutions for the final task.
More recently, the Forward-Forward algorithm \citep{hinton2022forward} was proposed as a potential alternative inspired by biology. It discards backpropagation entirely, using two forward passes—one with positive (real) data and one with negative data—and updating weights based on a local goodness metric specific to each layer. The objective is layer-local: maximizing goodness for positive samples and minimizing it for negative samples. This avoids backpropagation but requires generating negative data and doubles the computation per update compared to a single forward pass. Its scalability and performance across diverse tasks are still under investigation.
Further exploring alternatives to full backpropagation, \citet{radhakrishnan2024mechanism} recently introduced the Average Gradient Outer Product (AGOP) as a backpropagation-free mathematical mechanism to characterize and enable feature learning. Their work demonstrates that AGOP captures learned features across diverse architectures and can instill feature learning capabilities in models like kernel machines, notably through their Recursive Feature Machine (RFM) algorithm.

\subsection{Alternative mathematical optimization frameworks}
Beyond heuristic or bio-inspired approaches, alternative mathematical optimization frameworks have been applied to neural network training, often by reformulating the learning problem. The Alternating Direction Method of Multipliers (ADMM)~\citep{boyd2011distributed} is a general framework for constrained optimization that decomposes a large problem into smaller, potentially easier subproblems that are solved iteratively. It has been explored for training neural networks, sometimes by introducing auxiliary variables and constraints to decouple layers or enforce structure \citep{glowinski1975approximation, taylor2016training}. The Method of Auxiliary Coordinates (MAC) represents another such decomposition strategy for deeply nested systems~\citep{carreira2014distributed}. The objective in these frameworks is typically the original global loss, subject to reformulations. ADMM can handle non-differentiable regularizers and constraints but requires careful problem formulation, and solving the subproblems efficiently can be challenging, often involving overhead from dual variable updates or the solution of complex subproblems like large matrix inversions~\citep{taylor2016training}.
Block Coordinate Descent (BCD) methods optimize the network parameters block by block (e.g., layer by layer) while keeping other parameters fixed \citep{bertsekas1997nonlinear}. Each subproblem optimizes the global loss with respect to a subset of variables. BCD can be simpler to implement and potentially more memory-efficient than SGD for certain structures, but its convergence can be slow, dependent on the block partitioning strategy, and it may struggle with highly correlated parameters.

Crucially, the perspective of training as a \textbf{feasibility problem} (finding parameters $\vtheta$ such that $\vf(\vx_i, \vtheta) = \vy_i$ for all $i$) provides a distinct reformulation suitable for iterative projection methods like AP or DR. This ``learning without loss'' approach was conceptually explored by \citet{elser2021learning}, who proposed using the Difference Map algorithm (related to DR) to find feasible points in the intersection of constraints derived from individual data samples. Elser demonstrated the concept with illustrative examples. The projection-based method detailed in our work builds directly on this feasibility perspective, leveraging a fine-grained decomposition based on the computation graph's primitive functions, which allows for efficient, parallelizable projection steps.

\subsection{Symbolic, logic-based, and combinatorial optimization approaches}
A less common direction treats network training or design as a discrete or symbolic problem. By formulating aspects of a neural network's parameters or structure (e.g., activation functions, connectivity) as a discrete optimization problem, one can sometimes leverage mature combinatorial solvers like SAT solvers or integer programming techniques. These approaches often arise in the context of ``neuro-symbolic AI,'' aiming to integrate neural learning with symbolic reasoning. The objective might be to find a network satisfying certain logical constraints or optimizing a discrete objective. However, these methods often suffer from severe scalability challenges as network dimensions grow and are typically applicable only to specific problem types or network architectures.

\section{PJAX: A PROJECTION-BASED NUMERICAL COMPUTATION FRAMEWORK}
\label{sec:pjax}

Building upon the reformulation of neural network training as a feasibility problem solvable via projection methods (\cref{sec:method}), we introduce PJAX, a numerical computation framework designed to implement and solve such problems efficiently. PJAX aims to be an analogue of modern automatic differentiation (autodiff) libraries like Theano~\citep{team2016theano}, TensorFlow~\citep{abadi2016tensorflow}, PyTorch~\citep{paszke2019pytorch}, and JAX~\citep{jax2018}. However, unlike autodiff libraries that compute gradients, PJAX's core computational mechanism revolves around projection operators. By leveraging JAX as its backend (hence the name PJAX, with ``P'' signifying Projection), our framework inherits JAX's just-in-time (JIT) compilation capabilities and seamless execution across different hardware accelerators, including CPUs, GPUs, and TPUs.

The core design principle of PJAX mirrors that of autograd frameworks: users write numerical code using functions provided by the library, and the framework automatically tracks these computations to build an underlying representation—in our case, a computation graph suitable for projection-based methods rather than gradient calculation. An optimizer module then utilizes this graph and the associated projection operators to find feasible solutions to the defined problem. The primary goal is to offer the core functionality of a projection-based computation framework with an interface familiar to users of JAX or similar libraries, complemented by a high-level API for streamlined machine learning model definition.

\subsection{Core components}
PJAX is built upon core components implemented using JAX primitives. Accessed only indirectly via the high-level API, these components form the foundational building blocks, implementing the essential primitive functions and projection operators required for the optimization process:

\paragraph{Primitive functions} correspond to the hidden function nodes in the computation graph described in \cref{sec:computation_graph}. Examples include fundamental operations like \texttt{identity}, \texttt{add}, \texttt{dot}, \texttt{sum\_relu}, \texttt{max}, and \texttt{quantize}. For each primitive function $f$, we implement both its standard forward evaluation and the corresponding projection operator $\proj_{\graph(f)}$. Detailed definitions and derivations of these projection operators are included in \cref{sec:primitive_projections}.

\paragraph{Loss functions} correspond to the target nodes in the computation graph. For now, we provide a \texttt{cross\_entropy} and \texttt{margin\_loss} loss for classification tasks. Further details can be found in \cref{sec:output_operators}.

\paragraph{Shape transformations,} such as \texttt{index}, \texttt{reshape}, \texttt{transpose}, \texttt{repeat}, \texttt{concat}, \texttt{padding}, and \texttt{conv\_patch}, are treated as special \emph{no-ops} nodes that manipulate the shape or layout of data without altering its numerical values. These transformations must be invertible, and PJAX implements both the forward and inverse operations. These transformations do not impose feasibility constraints themselves (i.e., they do not define sets $\sC_v$ for projection). Instead, when a projection operator $\proj_{\sC_v}$ associated with a computational node $v$ needs to access values from incoming edges or distribute values to outgoing edges, these values are automatically passed through any intermediate shape transformation nodes using their forward or inverse implementations as needed. This design allows users to incorporate complex tensor manipulations common in neural networks without requiring additional projection operators.

\subsection{User API}
Users interact with PJAX through a high-level API designed to closely resemble the JAX / NumPy interface, facilitating adoption for those already familiar with these libraries.

\paragraph{Computation container} The \texttt{Computation} class is fundamental for objects managed by PJAX. A \texttt{Computation} instance can hold input data (as an \texttt{Array} or \texttt{Parameter}) or represent the symbolic output of a PJAX operation, maintaining references to the operation and its inputs. This mechanism allows PJAX to trace the sequence of operations and construct the computation graph implicitly.

\paragraph{Data containers} include \texttt{Array} (for constant inputs) and \texttt{Parameter} (for variables to be optimized). Both are subclasses of \texttt{Computation} and act as wrappers around standard JAX arrays (\texttt{jax.numpy.ndarray}).

\paragraph{API functions} provided by PJAX operate on \texttt{Computation} objects and return new \texttt{Computation} objects, thereby extending the computation graph. This includes wrappers for the core primitive functions, output constraints, and shape transformations. Furthermore, PJAX offers functions designed to mirror the \texttt{jax.numpy} API, such as \texttt{dot} (with batching semantics), \texttt{matmul}, \texttt{swapaxes}, \texttt{moveaxis}, \texttt{expand\_dims}, \texttt{squeeze}, \texttt{stack}, etc. These higher-level functions are implemented internally using combinations of PJAX core functions and the \texttt{vmap} utility, obviating the need to define unique projection operators for each one.

\paragraph{Vectorization (\texttt{vmap})} is supported through a utility, \texttt{pjax.vmap}, with a signature and semantics analogous to \texttt{jax.vmap}. It is used internally to implement batch-aware operations (e.g., \texttt{matmul} from \texttt{dot}) and is exposed to the user, enabling automatic vectorization of user-defined functions composed of PJAX operations. This is critical for achieving high performance on modern hardware.

\subsection{Optimizer module (\texttt{pjax.optim})}
The \texttt{pjax.optim} module contains the algorithms that solve the feasibility problem defined by the computation graph. Currently available optimizers include Alternating Projections (AP), Cyclic Projections (CP), Douglas–Rachford (DR), and the Difference Map algorithm \citep{elser2003phase}. The optimizer's \texttt{update} function accepts:
\begin{itemize}
    \item A user-defined Python function that computes the desired output using PJAX operations. This function implicitly defines the computation graph and constraints. It is analogous to the function one might pass to \texttt{jax.grad}, often representing the forward pass of a model, $\vtheta \mapsto \vf(\vx, \vtheta)$, whose output is then constrained according to the graph (e.g., via Output Constraint nodes).
    \item A dictionary (or other \emph{pytree} structure) holding the current state of the \texttt{Parameter} objects (e.g., network weights).
\end{itemize}
The optimizer then applies the function to the parameters, computes and stores intermediate outputs, forms a bipartition of the graph, and performs a specified number of update steps using the chosen projection algorithm. Finally, it returns a dictionary containing the updated consensus values for the parameters.

\subsection{High-level neural network API (\texttt{pjax.nn})}
Inspired by Flax, \texttt{pjax.nn} simplifies model definition and training via a compositional \texttt{Module} interface for reusable components (e.g., layers, blocks). It automates parameter handling (initialization, naming, sharing), enabling users to define complex architectures with familiar concepts, which are readily compatible with the \texttt{pjax.optim} module for projection-based training.

\subsection{Example: MLP definition and computation graph}
To illustrate how a neural network is defined using the \texttt{pjax.nn} API and how PJAX subsequently constructs a detailed computation graph, we consider an MLP with a single hidden layer and ReLU activation. Mathematically, this MLP is defined as
\begin{equation}
    \label{eq:mlp_example_definition}
    \vf(\vx; \mW^{\text{hidden}}, \vb^{\text{hidden}}, \mW^{\text{out}}) = \mW^{\text{out}} \operatorname{ReLU} (\mW^{\text{hidden}} \vx + \vb^{\text{hidden}})
\end{equation}
where $\vx$ is the input, $\mW^{\text{hidden}}$ and $\vb^{\text{hidden}}$ are the weights and biases for the hidden layer, respectively, and $\mW^{\text{out}}$ represents the weights for the output layer.
The corresponding PJAX code for this MLP is shown in \cref{lst:mlp_example_pjax}. 

\begin{minipage}{\textwidth}
    \begin{lstlisting}[language=Python, caption={Python code for an MLP using the \texttt{pjax} API, including training loop}, label={lst:mlp_example_pjax}]
    import jax
    import pjax
    from pjax import nn, optim
    
    class MLP(nn.Module):
        def __init__(self, in_features, hidden_features, num_classes):
            super().__init__()
            self.hidden = nn.Linear(in_features, hidden_features)
            self.relu = nn.ReLU(hidden_features) 
            self.out = nn.Linear(hidden_features, num_classes)
    
        def __call__(self, x):
            x = self.hidden(x)
            x = self.relu(x)
            x = self.out(x)
            return x
    
    
    model = MLP(in_features=784, hidden_features=256, num_classes=10)
    params = model.init(jax.random.key(0))
    optimizer = optim.DouglasRachford(steps_per_update=50)
    
    for (x, y) in dataloader:
        def loss_fn(params):
            logits = model.apply(params, x)
            return nn.cross_entropy(logits, y)
    
        params, loss = optimizer.update(loss_fn, params)
    \end{lstlisting}
\end{minipage}
When this \texttt{MLP} model is instantiated and applied to an input batch (e.g., 32 MNIST vectors, each $784$-dimensional, with a hidden layer of $16$ features and $10$ output classes), PJAX traces the operations. The resulting computation graph, shown in \cref{fig:pjax_mlp_graph}, visualizes this trace.

In this graph, nodes represent PJAX components. \texttt{Array} nodes hold constant data like the input batch and target data for the \texttt{cross\_entropy} loss. Learnable parameters, whose names are assigned by our \texttt{pjax.nn} API, are represented as \texttt{Parameter} nodes; for this MLP, these include \texttt{hidden.weight} (corresponding to $\mW^{\text{hidden}}$), \texttt{relu.bias} (corresponding to $\vb^{\text{hidden}}$ and applied at the \texttt{sum\_relu} stage), and \texttt{out.weight} (corresponding to $\mW^{\text{out}}$). The primitive functions shown are \texttt{dot} (scalar dot product) and \texttt{sum\_relu}. Shape transformation nodes, specifically \texttt{reshape} and \texttt{repeat}, prepare input tensors for the batched scalar operations. For instance, to implement the batched matrix multiplication in the first linear layer ($\mW^{\text{hidden}} \vx$), the input \texttt{Array (32,784)} and the \texttt{hidden.weight (784,16)} parameter are reshaped and repeated to match dimensions for the \texttt{dot} operation.

This graph details the concrete sequence of operations as traced by PJAX. In contrast, a conceptual illustration (as in \cref{fig:dag}) would represent each neuron's compound operations (dot product and activation) as distinct explicit nodes with edges fanning out to all neurons in the subsequent layer, rather than utilizing PJAX's explicit shape transformations for vectorized scalar primitives. Another minor technical difference is that the conceptual graph does not include a node for the target data, rather, it's implicitly represented in the loss function.

\begin{figure}[htb]
    \centering
    \includegraphics[width=1.0\textwidth]{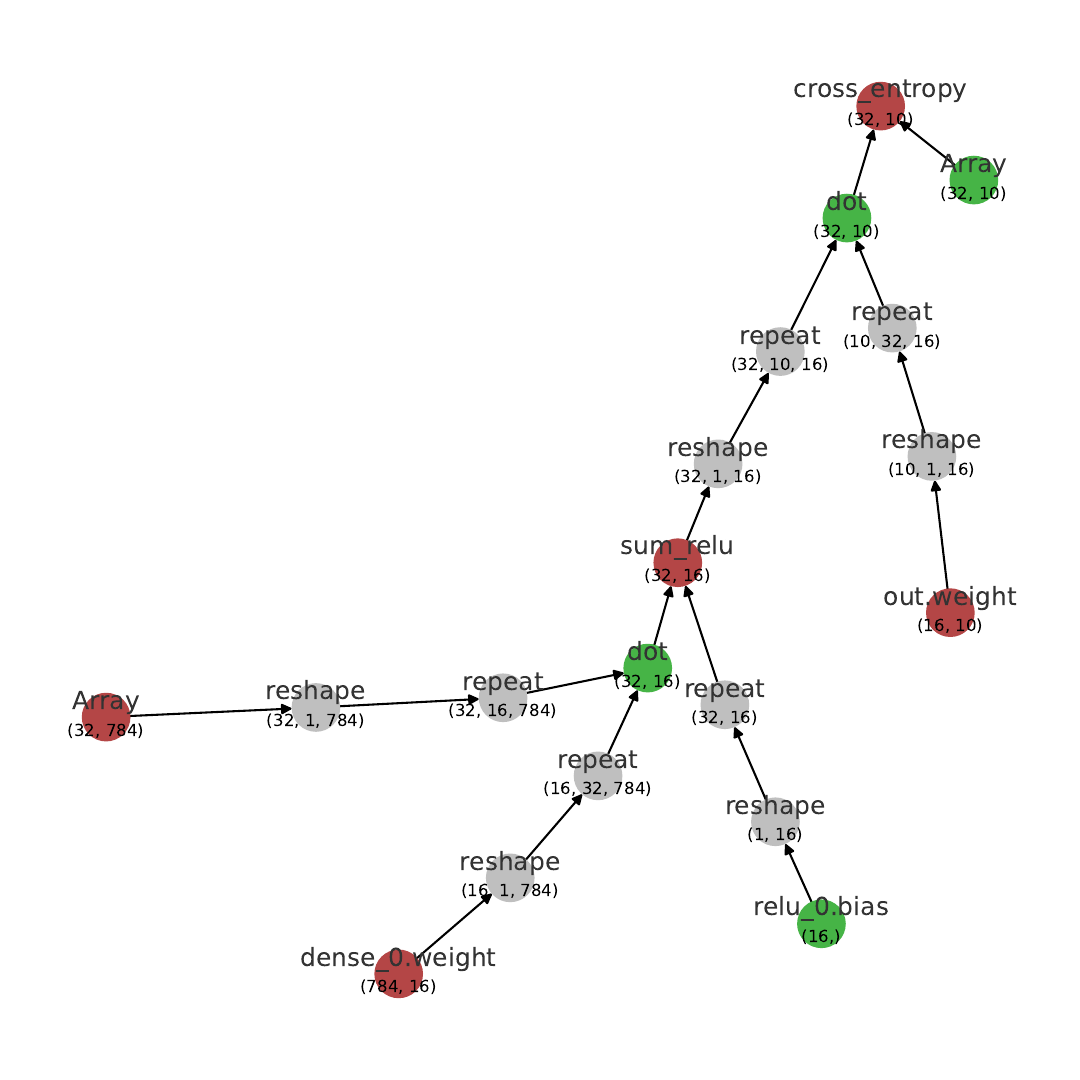}
    \caption{PJAX computation graph for the $16$-neuron hidden layer MLP (\cref{lst:mlp_example_pjax}), processing a batch of $32$ MNIST samples. Node sublabels indicate tensor shapes.}
    \label{fig:pjax_mlp_graph}
\end{figure}

\clearpage
\section{PROJECTION OPERATORS}
\label{sec:projections}
Here we provide the mathematical details underpinning the projection steps central to our method (\cref{sec:method}). These steps form the computational core of each iteration within the projection algorithms used for training (\cref{alg:training}). We begin by presenting the theorem that justifies how projections onto hidden function node constraints are computed by leveraging projections onto the graphs of the underlying primitive functions (\cref{thm:consensus-constrained-projection}). Subsequently, in \cref{sec:primitive_projections}, we detail the specific orthogonal projection operators onto the graphs ($\graph(f)$) for various primitive functions commonly used in neural networks, such as linear operations, activations, and pooling. Finally, \Cref{sec:output_operators} describes the operators employed at the output nodes to enforce conditions derived from the learning objective.

\begin{theorem}[Projection onto consensus sets]
    \label{thm:consensus-constrained-projection}
    Let $\bar{\sC} \subseteq \R^{d_x} \times \R^{d_y}$ be a non-empty closed set (e.g., the graph of a primitive function). The associated consensus set $\sC \subseteq \R^{d_x} \times (\R^{d_y})^n$ is
    \begin{equation}
        \sC = \left\{ (\vw, \vu_1, \dots, \vu_n) \mid (\vw, \vu_i) \in \bar{\sC} \forall i=1,\dots,n \right\}.
    \end{equation}
    Given a point $\vz_0 = (\vx_0, \vy_{0,1}, \dots, \vy_{0,n})$, a projection onto $\sC$ can be computed as follows:
    \begin{enumerate}
        \item Compute the average $\vy$ component: $\bar{\vy}_0 = \frac{1}{n} \sum_{i=1}^{n} \vy_{0,i}$.
        \item Choose a minimizer of the weighted projection problem for the pair $(\vx_0, \bar{\vy}_0)$ onto the base set $\bar{\sC}$:
            \begin{equation} \label{eq:weighted_proj_consensus}
            (\vx, \bar{\vy}) \in \argmin_{(\vw, \bar{\vu}) \in \bar{\sC}} \left( \norm{\vw - \vx_0}^2 + n \norm{\bar{\vu} - \bar{\vy}_0}^2 \right).
            \end{equation}
        \item A projection onto the consensus set $\sC$ is formed by replicating the $\bar{\vy}$ component:
            \begin{equation}
             (\vx, \underbrace{\bar{\vy}, \dots, \bar{\vy}}_{n \text{ times}}).
            \end{equation}
    \end{enumerate}
    This theorem justifies the projection procedure for function nodes described in \cref{sec:method}. In that context, $\bar{\sC} = \graph(f_v)$, $\vx_0$ holds incoming edge values $(z_{uv})$, $\vy_{0,i}$ hold outgoing values $z_{vw}$, and $n=|\sN^+(v)|$. Step 2 yields the updated incoming values $\vx$ (new $(z_{uv})$) and the single consensus outgoing value $\bar{\vy}$ (new $z_{\text{out}}$). Step 3 constructs the state update for $\vz$.
    
    For simplicity, our implementation uses a standard projection onto $\bar{\sC}$ in Step 2, rather than the theoretically derived weighted projection. While this alters the exact projection dynamics, any resulting fixed point still satisfies both the consensus requirement and the base constraint $(\vx, \bar{\vy}) \in \bar{\sC}$.
\end{theorem}

\begin{proof}
    Let $\vz_0 = (\vx_0, \vy_{0,1}, \dots, \vy_{0,n})$. We seek a projection of $\vz_0$ onto $\sC$, i.e.\ a point $(\vx, \vy_1, \dots, \vy_n) \in \sC$ minimizing the squared Euclidean distance $f$ to $\vz_0$. We use $(\vw, \vu_1, \dots, \vu_n)$ as dummy variables for points in $\sC$
    \begin{equation}
     \min_{(\vw, \vu_1, \dots, \vu_n) \in \sC} \left( \norm{\vw - \vx_0}^2 + \sum_{i=1}^n \norm{\vu_i - \vy_{0,i}}^2 \right).
    \end{equation}
    The constraint $(\vw, \vu_i) \in \bar{\sC}$ for all $i$ implies that any feasible point must satisfy $\vu_1 = \dots = \vu_n = \bar{\vu}$ for some $\bar{\vu}$ where $(\vw, \bar{\vu}) \in \bar{\sC}$.
    The problem reduces to finding $(\vw, \bar{\vu}) \in \bar{\sC}$ that minimizes
    \begin{equation}
     f(\vw, \bar{\vu}) = \norm{\vw - \vx_0}^2 + \sum_{i=1}^n \norm{\bar{\vu} - \vy_{0,i}}^2.
    \end{equation}
    Let $\bar{\vy}_0 = \frac{1}{n} \sum_{i=1}^n \vy_{0,i}$. Using the identity $\sum_{i=1}^n \norm{a - b_i}^2 = n \norm{a - \bar{b}}^2 + \sum_{i=1}^n \norm{\bar{b} - b_i}^2$, the sum term becomes
    \begin{equation}
     \sum_{i=1}^n \norm{\bar{\vu} - \vy_{0,i}}^2 = n \norm{\bar{\vu} - \bar{\vy}_0}^2 + \sum_{i=1}^n \norm{\bar{\vy}_0 - \vy_{0,i}}^2.
    \end{equation}
    Substituting this into $f(\vw, \bar{\vu})$, the minimization problem is equivalent (since the last term is constant w.r.t. $\vw, \bar{\vu}$) to
    \begin{equation}
     \min_{(\vw, \bar{\vu}) \in \bar{\sC}} \left( \norm{\vw - \vx_0}^2 + n \norm{\bar{\vu} - \bar{\vy}_0}^2 \right).
    \end{equation}
    Any minimizer $(\vx, \bar{\vy})$ of \cref{eq:weighted_proj_consensus} yields a projection onto $\sC$.
    Such a projection is constructed by replicating $\bar{\vy}$, yielding $(\vx, \underbrace{\bar{\vy}, \dots, \bar{\vy}}_{n \text{ times}})$, as stated in step 3.
\end{proof}

\subsection{Projection operators for primitive functions}
\label{sec:primitive_projections}
Below, we detail and derive the projection operators onto the graphs of several primitive functions used within the PJAX framework. Recall that for a function $f: \R^n \to \R^m$, its graph is defined as $\graph(f) = \{ (\vx, f(\vx)) \mid \vx \in \R^n \}$. The projection onto this graph, $\proj_{\graph(f)}(\vx_0, \vy_0)$, finds the point on the graph closest to $(\vx_0, \vy_0)$. 

\begin{table}[h!]
    \centering
    \caption{Summary of primitive functions presented.}
    \label{tab:primitive_function_summary}
    \begin{tabular}{lll}
    \toprule
    Function & Definition & Theorem Reference \\
    \midrule
    Identity & $f(x) = x$ & \cref{thm:identity} \\
    Sum & $f(\vx) = \vone^\top \vx$ & \cref{thm:sum} \\
    ReLU of summation & $f(\vx) = \max\{0, \vone^\top \vx\}$ & \cref{thm:sumrelu} \\
    Dot Product & $f(\vx, \vy) = \langle \vx, \vy \rangle$ & \cref{thm:dotproduct} \\
    Maximum & $f(\vx) = \max_i \{x_i\}$ & \cref{thm:maximum} \\
    Quantization & $f(x)$ maps $x$ to nearest in $Z$ & \cref{thm:quantization} \\
    \bottomrule
    \end{tabular}
\end{table}

\begin{theorem}[Identity]
    \label{thm:identity}
    The projection operator onto the graph of the identity function $\operatorname{id}(x) = x$ ($x \in \R$) is given by an average
    \begin{equation}
        \proj_{\graph(\operatorname{id})}(x_0, y_0) = \left( \frac{x_0 + y_0}{2}, \frac{x_0 + y_0}{2} \right).
    \end{equation}
\end{theorem}
\begin{proof}
    We seek the point $(x, x)$ on the line $y=x$ that is closest to $(x_0, y_0)$. This involves minimizing the squared distance $(x - x_0)^2 + (x - y_0)^2$. Standard calculus techniques (setting the derivative with respect to $x$ to zero) or geometric reasoning show that the minimum occurs when $x = (x_0 + y_0) / 2$. Since the graph is a line (closed and convex), the projection is unique.
\end{proof}

\begin{theorem}[Sum]
    \label{thm:sum}
    Let $\operatorname{sum}: \R^n \to \R$ be the summation function $\operatorname{sum}(\vx) = \vone^\top \vx$. The projection of $(\vx_0, y_0) \in \R^n \times \R$ onto its graph is
    \begin{equation}
        \proj_{\graph(\operatorname{sum})}(\vx_0, y_0) = \left(\vx_0 + \lambda \vone, y_0 - \lambda\right), \quad \text{where} \quad \lambda = \frac{y_0 - \vone^\top \vx_0}{n + 1}.
    \end{equation}
\end{theorem}
\begin{proof}
    We seek the point $(\vx, y)$ on the hyperplane $y = \vone^\top \vx$ that minimizes the squared distance $\norm{\vx - \vx_0}^2 + (y - y_0)^2$. Geometrically, the vector connecting $(\vx_0, y_0)$ to its projection $(\vx, y)$ must be orthogonal to the hyperplane. The normal vector to the hyperplane $y - \vone^\top \vx = 0$ is $(-\vone, 1)$. Thus, $(\vx - \vx_0, y - y_0)$ must be proportional to $(-\vone, 1)$. Setting $(\vx - \vx_0, y - y_0) = -\lambda (-\vone, 1)$ gives $\vx = \vx_0 + \lambda \vone$ and $y = y_0 - \lambda$. Substituting into the hyperplane equation $y = \vone^\top \vx$ allows solving for $\lambda = (y_0 - \vone^\top \vx_0) / (n + 1)$. The graph is an affine subspace (closed and convex), guaranteeing a unique projection.
\end{proof}

\begin{theorem}[ReLU of summation]
    \label{thm:sumrelu}
    Define $\operatorname{SumReLU}(\vx) = \max\{0, \vone^\top \vx\}$ for $\vx \in \R^n$. The graph $\sC = \graph(\operatorname{SumReLU})$ consists of two parts ($\sC = \sC_1 \cup \sC_2$):
    \begin{enumerate}
        \item The ``flat'' region $\sC_1 = \{ (\vx, y) \mid \vone^\top \vx \le 0, y=0 \}$.
        \item The ``sloped'' region $\sC_2 = \{ (\vx, y) \mid \vone^\top \vx \ge 0, y=\vone^\top \vx \}$.
    \end{enumerate}
    The projection of a point $(\vx_0, y_0) \in \R^n \times \R$ onto the graph is found by projecting onto these two regions and selecting the candidate closest to $(\vx_0, y_0)$:
    \begin{enumerate}
        \item Candidate 1, $(\vx^{(1)}, y^{(1)}) = \proj_{\sC_1}(\vx_0, y_0)$, is given by
        \begin{equation}
            \vx^{(1)} = \vx_0 - \max\left\{0, \frac{\vone^\top \vx_0}{n}\right\} \vone, \qquad y^{(1)} = 0.
        \end{equation}
        \item Candidate 2, $(\vx^{(2)}, y^{(2)}) = \proj_{\sC_2}(\vx_0, y_0)$, is computed by first projecting onto the hyperplane $y=\vone^\top \vx$:
        \begin{equation}
            (\hat{\vx}, \hat{y}) = (\vx_0 + \lambda\vone, y_0 - \lambda), \quad \text{where } \lambda = \frac{y_0 - \vone^\top \vx_0}{n+1}.
        \end{equation}
        Then, the projection onto $\sC_2$ is determined by whether $(\hat{\vx}, \hat{y})$ satisfies the non-negativity constraint
        \begin{equation}
            (\vx^{(2)}, y^{(2)}) = \begin{cases} 
                (\hat{\vx}, \hat{y}) & \text{if } \vone^\top \hat{\vx} \ge 0 \\ 
                \left( \vx_0 - \frac{\vone^\top \vx_0}{n} \vone, 0 \right) & \text{otherwise.}
            \end{cases}
        \end{equation}
    \end{enumerate}
    Finally, select
    \begin{equation}
        \proj_{\graph(\operatorname{SumReLU})}(\vx_0, y_0) = \argmin_{(\vx, y) \in \{(\vx^{(1)}, y^{(1)}), (\vx^{(2)}, y^{(2)})\}} \norm{(\vx, y) - (\vx_0, y_0)}^2.
    \end{equation}
\end{theorem}
\begin{proof}
    The graph $\mathcal{C} = \graph(\operatorname{SumReLU})$ is the union of two closed convex sets, $\mathcal{C}_1$ and $\mathcal{C}_2$. The projection $\proj_{\sC}(\vx_0, y_0)$ is therefore the point closer to $(\vx_0, y_0)$ among the projections onto $\sC_1$ and $\sC_2$. These projections minimize the squared Euclidean distance from $(\mathbf{x}_0, y_0)$ subject to the constraints defining each set. Solving this minimization, for instance using the method of Lagrange multipliers, yields the formulas presented in the theorem.
\end{proof}

\begin{theorem}[Dot product]
\label{thm:dotproduct}
    Let $\operatorname{dot}(\vx, \vy) = \langle \vx, \vy \rangle$ for $\vx, \vy \in \R^n$. For a point $(\vx_0, \vy_0, z_0) \in \R^n \times \R^n \times \R$ with $\vx_0 \neq \pm \vy_0$, the orthogonal projection onto the graph $\graph(\operatorname{dot})$ is unique and given by
    \begin{equation}
        \proj_{\graph(\operatorname{dot})}(\vx_0, \vy_0, z_0) = \left( \vx(\lambda), \vy(\lambda), z_0 - \lambda \right),
    \end{equation}
    where
    \begin{equation}
        \vx(\lambda) = \frac{\vx_0 + \lambda \vy_0}{1 - \lambda^2}, \quad \vy(\lambda) = \frac{\vy_0 + \lambda \vx_0}{1 - \lambda^2}.
    \end{equation}
    The scalar parameter $\lambda$ is the unique root in $]-1, 1[$ of the function
    \begin{equation} \label{eq:dot_product_lambda_func}
        f(\lambda) = \frac{(1 + \lambda^2) p + \lambda q}{(1 - \lambda^2)^2} - z_0 + \lambda = 0,
    \end{equation}
    with $p = \langle \vx_0, \vy_0 \rangle$ and $q = \norm{\vx_0}^2 + \norm{\vy_0}^2$.
    This root can be efficiently found using Newton's method. Practically, 5 to 10 iterations starting from $\lambda=0$ are sufficient to converge to a solution with high accuracy.
\end{theorem}
\begin{proof}
    We seek the point $(\vx, \vy, z)$ on the graph $z = \langle \vx, \vy \rangle$ that minimizes the squared Euclidean distance $\norm{\vx - \vx_0}^2 + \norm{\vy - \vy_0}^2 + (z - z_0)^2$. This is equivalent to minimizing the unconstrained function
    \begin{equation}
        g(\vx, \vy) = \norm{\vx - \vx_0}^2 + \norm{\vy - \vy_0}^2 + (\langle \vx, \vy \rangle - z_0)^2
    \end{equation}
    over $(\vx, \vy) \in \R^n \times \R^n$. Existence of a minimum is guaranteed as $g$ is continuous and coercive. The first-order optimality conditions are $\nabla_{\vx} g = \vzero$ and $\nabla_{\vy} g = \vzero$. Let $\lambda = z_0 - \langle \vx, \vy \rangle$. The conditions simplify to
    \begin{equation}
        \vx = \vx_0 + \lambda \vy \quad \text{and} \quad \vy = \vy_0 + \lambda \vx.
    \end{equation}
    Solving this linear system for $\vx$ and $\vy$ (assuming $\lambda^2 \neq 1$, which holds for the minimum when $\vx_0 \neq \pm \vy_0$) yields
    \begin{equation}
        \vx(\lambda) = \frac{\vx_0 + \lambda \vy_0}{1 - \lambda^2}, \quad \vy(\lambda) = \frac{\vy_0 + \lambda \vx_0}{1 - \lambda^2}.
    \end{equation}
    Substituting these back into the definition $\lambda = z_0 - \langle \vx(\lambda), \vy(\lambda) \rangle$ and simplifying leads to the condition $f(\lambda) = 0$ as defined in \cref{eq:dot_product_lambda_func}, where $p = \langle \vx_0, \vy_0 \rangle$ and $q = \norm{\vx_0}^2 + \norm{\vy_0}^2$.
    Analogous to the analysis for projections onto hyperbolas, the function $f(\lambda)$ has a unique root in $]-1, 1[$ when $\vx_0 \neq \pm \vy_0$, ensuring a unique projection~\citep{bauschke2022projecting}. 
    The $z$ component of the projection is $z = \langle \vx(\lambda), \vy(\lambda) \rangle = z_0 - \lambda$.
\end{proof}

\begin{theorem}[Maximum]
\label{thm:maximum}
    Let $\operatorname{max}: \R^n \to \R$ be the maximum function $\operatorname{max}(\vx) = \max_i \{x_i\}$. To project $(\vx_0, y_0)$ onto the graph $y = \operatorname{max}(\vx)$, first sort $\vx_0$ such that $x_{0,1} \le \ldots \le x_{0,n}$ (and note the permutation used).
    Then, for $k = 1, \ldots, n$, generate candidate points $(\vx^{(k)}, y^{(k)})$ based on the hypothesis that the maximum $y$ is achieved by components $x_k, \ldots, x_n$:
    \begin{equation}
        y^{(k)} = \frac{\left(\sum_{i=k}^n x_{0,i}\right) + y_0}{n - k + 2}
        \qquad \text{and} \qquad
        \vx^{(k)}_i = \begin{cases}
            y^{(k)}, & \text{if } i \ge k \\
            x_{0,i}, & \text{if } i < k
        \end{cases}
    \end{equation}
    A candidate $(\vx^{(k)}, y^{(k)})$ is \emph{valid} if $\max(\vx^{(k)}) = y^{(k)}$. Given the sorted input, this simplifies to checking $x_{0, k-1} \le y^{(k)}$ (for $k>1$; $k=1$ is always valid).
    A projection is found by selecting a valid candidate $(\vx^{(k)}, y^{(k)})$ that minimizes the distance $\norm{(\vx^{(k)}, y^{(k)}) - (\vx_0, y_0)}^2$. The final vector $\vx$ must be permuted back to the original order.
\end{theorem}
\begin{proof}
    We minimize the squared distance $f(\vx, y) = \norm{\vx - \vx_0}^2 + (y - y_0)^2$ subject to $y = \max_i x_i$. The graph of the maximum function is closed, so a projection exists.
    Let $(\vx, y)$ be a projection, and let $I = \{i \mid x_i = y\}$ be the non-empty set of indices achieving the maximum. For $i \notin I$, optimality forces $x_i = x_{0,i}$; otherwise one could move $x_i$ toward $x_{0,i}$ while keeping $x_i \le y$ and decrease $f$. Hence $x_{0,i} \le y$ for all $i \notin I$.
    After sorting $\vx_0$, if there exist indices $i < j$ with $i \in I$ and $j \notin I$, then replacing $(x_i, x_j) = (y, x_{0,j})$ by $(x_{0,i}, y)$ preserves feasibility and does not increase the objective, since $x_{0,i} \le x_{0,j} \le y$. Repeating this exchange yields a projection whose active set is of the form $I = \{k, \dots, n\}$ for some $k \in \{1, \dots, n\}$.
    For such a suffix active set, we have $x_i = y$ for $i \ge k$ and $x_i = x_{0,i}$ for $i < k$. The objective thus reduces to minimizing
    \begin{equation}
        g(y) = \sum_{i=k}^n (x_{0,i} - y)^2 + (y_0 - y)^2
    \end{equation}
    with respect to $y$. Setting the derivative $g'(y)=0$ gives the value $y=y^{(k)}$ defined in the theorem, with corresponding candidate $\vx^{(k)}$.
    This candidate assumes the structure holds, which requires the validity check $x_{0, k-1} \le y^{(k)}$ (for $k>1$).
    These $n$ candidates cover all possible active sets of the form $I = \{k, \dots, n\}$ for a projection, so any valid candidate minimizing $f$ is a projection. The $k=1$ candidate is always valid, ensuring the candidate list is non-empty.
\end{proof}

\begin{theorem}[Quantization]
    \label{thm:quantization}
    Let $\operatorname{quant}: \R \to \R$ be the quantization function that maps a real number $x$ to the nearest point in a set $Z = \{z_1, z_2, \ldots, z_k\}$ of $k \ge 2$ equidistant points within $[-\alpha, \alpha]$, where $\alpha > 0$. Specifically, the points $z_i$ are given by
    \begin{equation}
        z_i = -\alpha + (i-1) \frac{2\alpha}{k-1} \quad \text{for } i=1, \dots, k.
    \end{equation}
    The function partitions the real line into intervals $I_i$ such that $\operatorname{quant}(x) = z_i$ for $x \in I_i$. These intervals are defined by the $k-1$ midpoints $m_i = (z_i + z_{i+1})/2$ for $i=1, \ldots, k-1$:
    \begin{equation}
        I_1 = (-\infty, m_1]
    \end{equation}
    \begin{equation}
        I_i = (m_{i-1}, m_i] \quad \text{for } i = 2, \ldots, k-1
    \end{equation}
    \begin{equation}
        I_k = (m_{k-1}, \infty)
    \end{equation}
    The graph of the quantization function is the union of horizontal line segments and rays
    \begin{equation}
        \graph(\operatorname{quant}) = \bigcup_{i=1}^{k} (I_i \times \{z_i\})
    \end{equation}
    This set is generally non-convex.

    To project a point $(x_0, y_0) \in \R \times \R$ onto this graph, first compute $k$ candidate points $(x^{(i)}, y^{(i)})$ by projecting $x_0$ onto each interval $I_i$ and keeping the corresponding $y$ value $z_i$:
    \begin{equation}
        (x^{(i)}, y^{(i)}) = (\proj_{I_i}(x_0), z_i) \qquad \text{for } i = 1, \ldots, k.
    \end{equation}
    The projection $\proj_{\graph(\operatorname{quant})}(x_0, y_0)$ is then given by the candidate $(x^{(i)}, y^{(i)})$ that is closest to $(x_0, y_0)$:
    \begin{equation}
        \proj_{\graph(\operatorname{quant})}(x_0, y_0) = \argmin_{i \in \{1,\ldots,k\}} \norm{(x^{(i)}, y^{(i)}) - (x_0, y_0)}^2.
    \end{equation}
    Note that since the graph is non-convex, the minimum distance might be achieved by multiple candidates; the $\argmin$ selects one such point (consistent with \cref{eq:proj}).
\end{theorem}
\begin{proof}
    We seek $(x, y) \in \graph(\operatorname{quant})$ that minimizes the squared distance $\norm{(x, y) - (x_0, y_0)}^2$. The graph is the union of closed, non-convex pieces, hence a minimum distance exists, but the projection is not necessarily unique.

    The minimizing point $(x, y)$ must belong to some piece $\sC_j = I_j \times \{z_j\}$. Consider the minimization restricted to an arbitrary but fixed piece $\sC_i = I_i \times \{z_i\}$. A point $(x, z_i)$ on this piece minimizes $\norm{(x, z_i) - (x_0, y_0)}^2$ subject to $x \in I_i$. This is equivalent to finding $x \in I_i$ that minimizes $(x - x_0)^2$, whose solution is $x = \proj_{I_i}(x_0)$.
    This identifies the candidate $(x^{(i)}, y^{(i)}) = (\proj_{I_i}(x_0), z_i)$ as the closest point on $\sC_i$ to $(x_0, y_0)$.

    The overall projection is the candidate $(x^{(j)}, y^{(j)})$ that yields the minimum squared distance $\norm{(x^{(i)}, y^{(i)}) - (x_0, y_0)}^2$ among all $i \in \{1, \ldots, k\}$, as stated in the theorem.
\end{proof}

\subsection{Output operators}
\label{sec:output_operators}
This subsection describes operators used at the output nodes (\cref{sec:method}) to enforce conditions derived from the task's loss function. We provide the projection operators for the margin loss constraint and the proximal operator for the cross-entropy loss function.

\begin{theorem}[Margin loss constraint]
    \label{thm:margin-loss}
    Consider a classification setting where the goal is to enforce a condition on a single logit output $x \in \R$ based on a label $y \in \R$. The margin loss constraint requires the logit to be non-positive for negative labels and greater than or equal to a positive margin $m > 0$ for positive labels.
    Let the constraint set $\sC_{y, m}$ be defined as
    \begin{equation}
        \sC_{y, m} = \begin{cases}
            (-\infty, 0] & \text{if } y \le 0 \\
            [m, \infty) & \text{otherwise.}
        \end{cases}
    \end{equation}
    The orthogonal projection of a point $x_0 \in \R$ onto this set is given by
    \begin{equation}
        \proj_{\sC_{y, m}}(x_0) = \begin{cases}
            \min(x_0, 0) & \text{if } y \le 0 \\
            \max(x_0, m) & \text{otherwise.}
        \end{cases}
    \end{equation}
\end{theorem}
\begin{proof}
    We seek $x \in \sC_{y, m}$ that minimizes $(x - x_0)^2$.
    Case 1: $y \le 0$. We need $x \in (-\infty, 0]$. If $x_0 \le 0$, then $x_0$ is already in the set, and the minimum distance (zero) is achieved at $x=x_0$. If $x_0 > 0$, the closest point in $(-\infty, 0]$ is $x=0$. Thus, the projection is $\min(x_0, 0)$.
    Case 2: $y > 0$. We need $x \in [m, \infty)$. If $x_0 \ge m$, then $x_0$ is in the set, and the minimum distance is achieved at $x=x_0$. If $x_0 < m$, the closest point in $[m, \infty)$ is $x=m$. Thus, the projection is $\max(x_0, m)$.
    Combining both cases yields the stated formula.
\end{proof}

\begin{theorem}[Proximal operator for cross-entropy loss]
    \label{thm:cross-entropy-prox}
    Let $\ell_{CE}: \R^d \times \R^d \to \R$ be the standard cross-entropy loss function for multi-class classification, defined for logits $\vx \in \R^d$ and a one-hot encoded target label vector $\vy \in \{0, 1\}^d$ (with $\vone^\top \vy = 1$) as
    \begin{equation}
        \ell_{CE}(\vx, \vy) = \log\left(\sum_{j=1}^d e^{x_j}\right) - \langle \vy, \vx \rangle.
    \end{equation}
    The proximal operator of this loss function, scaled by $\lambda > 0$, applied to a point $\vx_0 \in \R^d$ is defined as
    \begin{equation}
        \prox_{\lambda \ell_{CE}(\cdot, \vy)}(\vx_0) = \argmin_{\vx \in \R^d} \left( \lambda \ell_{CE}(\vx, \vy) + \frac{1}{2} \norm{\vx - \vx_0}^2 \right).
    \end{equation}
    The unique minimizer $\vx^* = \prox_{\lambda \ell_{CE}(\cdot, \vy)}(\vx_0)$ is characterized by the condition
    \begin{equation} \label{eq:prox_ce_fixed_point}
        \vx^* = \vx_0 + \lambda (\vy - \operatorname{softmax}(\vx^*)).
    \end{equation}
    This point $\vx^*$ can be found efficiently using iterative methods targeting this fixed-point equation, such as fixed-point iteration or Newton's method applied to the equivalent root-finding problem $\vx - \vx_0 - \lambda(\vy - \operatorname{softmax}(\vx)) = \vzero$. Practically, we found that a fixed-point iteration with 10 iterations provides a stable and efficient solution.
\end{theorem}
\begin{proof}
    The function $g(\vx) = \lambda \ell_{CE}(\vx, \vy) + \frac{1}{2} \norm{\vx - \vx_0}^2$ is strictly convex because $\ell_{CE}(\vx, \vy)$ is convex and $\frac{1}{2} \norm{\vx - \vx_0}^2$ is strictly convex. Therefore, a unique minimizer $\vx^*$ exists.
    The minimizer is characterized by the first-order optimality condition $\nabla g(\vx^*) = \vzero$.
    The gradient is
    \begin{equation}
        \nabla g(\vx) = \lambda \nabla_{\vx} \ell_{CE}(\vx, \vy) + \nabla_{\vx} \left(\frac{1}{2} \norm{\vx - \vx_0}^2\right).
    \end{equation}
    The gradient of the cross-entropy loss is $\nabla_{\vx} \ell_{CE}(\vx, \vy) = \operatorname{softmax}(\vx) - \vy$.
    The gradient of the quadratic term is $\vx - \vx_0$.
    Setting the total gradient to zero at $\vx^*$ gives
    \begin{equation}
        \lambda (\operatorname{softmax}(\vx^*) - \vy) + (\vx^* - \vx_0) = \vzero.
    \end{equation}
    Rearranging this equation yields the characterization
    \begin{equation}
        \vx^* = \vx_0 - \lambda (\operatorname{softmax}(\vx^*) - \vy) = \vx_0 + \lambda (\vy - \operatorname{softmax}(\vx^*)).
    \end{equation}
    This confirms \cref{eq:prox_ce_fixed_point}.
\end{proof}

\section{DETAILED EXPERIMENTAL RESULTS}
\label{sec:results_tables}
This section provides the detailed numerical results from the experiments described in \cref{sec:experiments}. We compare the performance of the projection-based methods (Douglas-Rachford (DR), Alternating Projections (AP), Cyclic Projections (CP)) against gradient-based baselines (Stochastic Gradient Descent (SGD), Adam) and gradient-free baselines (Feedback Alignment (FA) for MLPs and Forward–Forward (FF) for MLPs and CNNs). The results are presented in the tables below (\cref{tab:mlp_results} for MLPs, \cref{tab:cnn_results} for CNNs, and \cref{tab:rnn_results} for RNNs).

These results are obtained using default train/test splits (approximately 80\% train, 20\% test), with 10\% of training data reserved for validation. Early stopping is applied based on validation accuracy, with a patience of 5 checks (5000 steps) if no improvement is observed. Metrics are averaged over 5 independent runs with different random seeds for initialization and data shuffling; standard deviations ($\pm$) are reported where appropriate.
\begin{itemize}
    \item \textbf{Test Accuracy (\%):} The final classification accuracy achieved by the best model (selected based on validation performance) on the held-out test set. For the Shakespeare dataset, this refers to the accuracy of predicting the next character.
    \item \textbf{Train Accuracy (\%):} The final classification accuracy on the training set, reported for comparison with test accuracy.
    \item \textbf{Steps (k):} The number of training steps (parameter updates), reported in thousands ('k'), required to reach 99\% of the maximum validation accuracy achieved during that run. This metric indicates convergence speed in terms of updates.
    \item \textbf{Time (s):} The wall-clock time in seconds required to reach 99\% of the maximum validation accuracy. This measurement excludes time spent on data loading, validation checks, and initial JIT compilation, focusing on the core training loop execution time.
    \item \textbf{Time/Step (ms):} The average wall-clock time in milliseconds per single training step (parameter update step).
\end{itemize}

\begin{table}[htbp]
    \centering
    \caption{MLP Performance Comparison.}
    \label{tab:mlp_results}
    \resizebox{0.9\textwidth}{!}{ 
    \begin{tabular}{ @{} lllccccc @{} }
        \toprule
        Dataset & Architecture & Method & Test Acc. (\%) & Train Acc. (\%) & Steps (k) & Time (s) & Time/Step (ms) \\
        \midrule
        \multirow{21}{*}{MNIST} & \multirow{8}{*}{1 x 128} & SGD & 96.8 $\pm$ 0.1 & 97.6 $\pm$ 0.1 & 58.6 $\pm$ 7.2 & 171.4 $\pm$ 21.4 & 2.92 $\pm$ 0.01 \\
         & & Adam & 98.1 $\pm$ 0.1 & 100.0 $\pm$ 0.0 & 1.6 $\pm$ 0.5 & 4.7 $\pm$ 1.4 & 2.95 $\pm$ 0.01 \\
         & & FA & 97.6 $\pm$ 0.1 & 99.5 $\pm$ 0.2 & 8.0 $\pm$ 1.1 & 23.3 $\pm$ 3.1 & 2.92 $\pm$ 0.02 \\
         & & FF (SGD) & 12.7 $\pm$ 1.9 & 12.8 $\pm$ 2.2 & 1.6 $\pm$ 1.7 & 5.1 $\pm$ 5.6 & 3.12 $\pm$ 0.08 \\
         & & FF (Adam) & 95.6 $\pm$ 0.5 & 96.2 $\pm$ 0.6 & 82.4 $\pm$ 11.8 & 256.2 $\pm$ 35.2 & 3.12 $\pm$ 0.04 \\
         & & DR & 95.8 $\pm$ 0.6 & 96.9 $\pm$ 0.7 & 16.6 $\pm$ 3.3 & 3.7 $\pm$ 0.7 & 0.22 $\pm$ 0.00 \\
         & & AP & 91.5 $\pm$ 0.7 & 91.4 $\pm$ 0.8 & 25.8 $\pm$ 5.4 & 2.3 $\pm$ 0.5 & 0.09 $\pm$ 0.00 \\
         & & CP & 92.4 $\pm$ 0.6 & 92.5 $\pm$ 0.7 & 41.6 $\pm$ 11.4 & 8.6 $\pm$ 2.3 & 0.21 $\pm$ 0.00 \\
        \cmidrule(lr){2-8}
         & \multirow{8}{*}{4 x 128} & SGD & 96.9 $\pm$ 0.2 & 98.8 $\pm$ 0.2 & 32.4 $\pm$ 2.3 & 96.2 $\pm$ 7.2 & 2.97 $\pm$ 0.03 \\
         & & Adam & 98.1 $\pm$ 0.1 & 99.9 $\pm$ 0.1 & 1.8 $\pm$ 0.7 & 5.7 $\pm$ 2.3 & 3.09 $\pm$ 0.05 \\
         & & FA & 95.4 $\pm$ 0.8 & 96.1 $\pm$ 1.0 & 20.0 $\pm$ 6.2 & 60.2 $\pm$ 18.3 & 3.02 $\pm$ 0.03 \\
         & & FF (SGD) & 11.9 $\pm$ 0.8 & 11.8 $\pm$ 0.7 & 2.8 $\pm$ 1.7 & 9.0 $\pm$ 5.4 & 3.24 $\pm$ 0.05 \\
         & & FF (Adam) & 94.3 $\pm$ 0.4 & 94.2 $\pm$ 0.6 & 66.0 $\pm$ 9.3 & 220.3 $\pm$ 32.1 & 3.35 $\pm$ 0.06 \\
         & & DR & 85.5 $\pm$ 0.6 & 86.2 $\pm$ 0.9 & 29.6 $\pm$ 7.3 & 10.6 $\pm$ 2.6 & 0.36 $\pm$ 0.00 \\
         & & AP & 79.6 $\pm$ 0.5 & 79.2 $\pm$ 0.8 & 46.2 $\pm$ 4.2 & 9.4 $\pm$ 0.8 & 0.20 $\pm$ 0.00 \\
         & & CP & 79.9 $\pm$ 0.6 & 79.3 $\pm$ 0.8 & 47.8 $\pm$ 9.5 & 17.1 $\pm$ 3.5 & 0.36 $\pm$ 0.00 \\
        \cmidrule(lr){2-8}
         & \multirow{5}{*}{4 x 128 + Skip} & SGD & 96.6 $\pm$ 0.3 & 98.1 $\pm$ 0.6 & 28.2 $\pm$ 7.2 & 84.9 $\pm$ 21.9 & 3.01 $\pm$ 0.01 \\
         & & Adam & 98.1 $\pm$ 0.2 & 100.0 $\pm$ 0.1 & 1.4 $\pm$ 0.5 & 4.3 $\pm$ 1.6 & 3.08 $\pm$ 0.01 \\
         & & DR & 94.9 $\pm$ 0.4 & 96.1 $\pm$ 0.7 & 21.8 $\pm$ 6.0 & 9.3 $\pm$ 2.6 & 0.43 $\pm$ 0.00 \\
         & & AP & 90.3 $\pm$ 0.8 & 90.3 $\pm$ 0.8 & 32.6 $\pm$ 9.3 & 6.9 $\pm$ 2.0 & 0.21 $\pm$ 0.00 \\
         & & CP & 91.1 $\pm$ 0.5 & 91.1 $\pm$ 0.5 & 34.2 $\pm$ 3.5 & 18.7 $\pm$ 1.9 & 0.55 $\pm$ 0.00 \\
        \midrule
        \multirow{21}{*}{CIFAR-10} & \multirow{8}{*}{1 x 256} & SGD & 50.5 $\pm$ 0.1 & 67.0 $\pm$ 2.0 & 34.6 $\pm$ 4.8 & 307.1 $\pm$ 42.1 & 8.88 $\pm$ 0.06 \\
         & & Adam & 50.7 $\pm$ 0.8 & 80.9 $\pm$ 7.6 & 2.2 $\pm$ 0.4 & 20.2 $\pm$ 3.6 & 9.14 $\pm$ 0.09 \\
         & & FA & 49.9 $\pm$ 0.9 & 73.4 $\pm$ 2.9 & 5.8 $\pm$ 1.2 & 52.3 $\pm$ 10.6 & 9.08 $\pm$ 0.12 \\
         & & FF (SGD) & 9.9 $\pm$ 0.2 & 10.2 $\pm$ 0.4 & 2.6 $\pm$ 2.7 & 24.5 $\pm$ 25.3 & 9.29 $\pm$ 0.10 \\
         & & FF (Adam) & 45.4 $\pm$ 0.7 & 50.7 $\pm$ 1.3 & 100.0 $\pm$ 8.9 & 919.3 $\pm$ 83.5 & 9.19 $\pm$ 0.03 \\
         & & DR & 42.1 $\pm$ 1.1 & 55.6 $\pm$ 4.8 & 27.2 $\pm$ 10.3 & 62.3 $\pm$ 23.5 & 2.29 $\pm$ 0.00 \\
         & & AP & 38.5 $\pm$ 0.5 & 40.3 $\pm$ 0.5 & 11.0 $\pm$ 3.7 & 4.8 $\pm$ 1.6 & 0.43 $\pm$ 0.00 \\
         & & CP & 38.5 $\pm$ 0.5 & 40.3 $\pm$ 0.5 & 9.8 $\pm$ 3.1 & 14.4 $\pm$ 4.6 & 1.47 $\pm$ 0.00 \\
        \cmidrule(lr){2-8}
         & \multirow{8}{*}{4 x 256} & SGD & 47.8 $\pm$ 0.3 & 67.6 $\pm$ 2.6 & 21.6 $\pm$ 3.1 & 198.3 $\pm$ 27.5 & 9.16 $\pm$ 0.10 \\
         & & Adam & 50.0 $\pm$ 0.2 & 79.3 $\pm$ 10.3 & 2.6 $\pm$ 1.4 & 24.4 $\pm$ 12.4 & 9.37 $\pm$ 0.05 \\
         & & FA & 29.3 $\pm$ 3.6 & 29.5 $\pm$ 3.4 & 9.0 $\pm$ 2.3 & 84.6 $\pm$ 21.1 & 9.34 $\pm$ 0.09 \\
         & & FF (SGD) & 11.0 $\pm$ 0.6 & 11.2 $\pm$ 0.6 & 1.2 $\pm$ 1.0 & 11.7 $\pm$ 9.2 & 9.48 $\pm$ 0.18 \\
         & & FF (Adam) & 43.6 $\pm$ 1.0 & 48.0 $\pm$ 2.5 & 92.6 $\pm$ 16.6 & 870.8 $\pm$ 155.3 & 9.40 $\pm$ 0.03 \\
         & & DR & 32.3 $\pm$ 0.9 & 35.4 $\pm$ 1.1 & 13.2 $\pm$ 5.5 & 39.8 $\pm$ 16.7 & 3.01 $\pm$ 0.00 \\
         & & AP & 32.5 $\pm$ 0.6 & 33.7 $\pm$ 0.9 & 13.6 $\pm$ 7.4 & 9.4 $\pm$ 5.1 & 0.69 $\pm$ 0.00 \\
         & & CP & 32.5 $\pm$ 0.5 & 33.8 $\pm$ 0.8 & 13.6 $\pm$ 7.4 & 27.0 $\pm$ 14.6 & 1.98 $\pm$ 0.00 \\
        \cmidrule(lr){2-8}
         & \multirow{5}{*}{4 x 256 + Skip} & SGD & 49.4 $\pm$ 0.3 & 68.7 $\pm$ 1.7 & 21.4 $\pm$ 2.4 & 200.5 $\pm$ 23.6 & 9.32 $\pm$ 0.08 \\
         & & Adam & 50.6 $\pm$ 0.2 & 89.2 $\pm$ 6.0 & 3.2 $\pm$ 1.0 & 30.0 $\pm$ 8.9 & 9.33 $\pm$ 0.20 \\
         & & DR & 39.2 $\pm$ 0.8 & 45.3 $\pm$ 1.9 & 20.8 $\pm$ 6.5 & 61.2 $\pm$ 19.1 & 2.94 $\pm$ 0.00 \\
         & & AP & 38.9 $\pm$ 0.6 & 41.0 $\pm$ 0.8 & 14.6 $\pm$ 3.9 & 10.2 $\pm$ 2.7 & 0.70 $\pm$ 0.00 \\
         & & CP & 38.9 $\pm$ 0.7 & 41.0 $\pm$ 0.7 & 14.6 $\pm$ 3.9 & 42.4 $\pm$ 11.3 & 2.91 $\pm$ 0.00 \\
        \midrule
        \multirow{21}{*}{Higgs} & \multirow{7}{*}{1 x 256} & SGD & 66.9 $\pm$ 0.3 & 66.9 $\pm$ 0.3 & 63.6 $\pm$ 11.4 & 55.5 $\pm$ 10.3 & 0.88 $\pm$ 0.01 \\
         & & Adam & 74.6 $\pm$ 0.1 & 74.6 $\pm$ 0.1 & 29.0 $\pm$ 3.9 & 27.3 $\pm$ 3.8 & 0.97 $\pm$ 0.01 \\
         & & FA & 71.3 $\pm$ 0.5 & 71.3 $\pm$ 0.5 & 59.0 $\pm$ 18.9 & 52.1 $\pm$ 17.1 & 0.89 $\pm$ 0.01 \\
         & & FF (SGD) & 57.6 $\pm$ 5.8 & 57.7 $\pm$ 5.7 & 152.8 $\pm$ 191.6 & 168.1 $\pm$ 210.9 & 1.07 $\pm$ 0.03 \\
         & & FF (Adam) & 71.1 $\pm$ 0.3 & 71.0 $\pm$ 0.4 & 55.6 $\pm$ 12.6 & 66.0 $\pm$ 15.5 & 1.18 $\pm$ 0.03 \\
         & & DR & 67.1 $\pm$ 0.4 & 67.0 $\pm$ 0.4 & 24.8 $\pm$ 5.1 & 0.9 $\pm$ 0.2 & 0.04 $\pm$ 0.00 \\
         & & AP & 65.4 $\pm$ 0.3 & 65.3 $\pm$ 0.3 & 27.4 $\pm$ 13.5 & 0.6 $\pm$ 0.3 & 0.02 $\pm$ 0.00 \\
         & & CP & 65.2 $\pm$ 0.4 & 65.2 $\pm$ 0.5 & 25.4 $\pm$ 10.8 & 0.7 $\pm$ 0.3 & 0.03 $\pm$ 0.00 \\
        \cmidrule(lr){2-8}
         & \multirow{7}{*}{4 x 256} & SGD & 71.1 $\pm$ 0.6 & 71.1 $\pm$ 0.6 & 178.8 $\pm$ 37.7 & 176.0 $\pm$ 38.7 & 0.98 $\pm$ 0.01 \\
         & & Adam & 76.6 $\pm$ 0.2 & 76.6 $\pm$ 0.2 & 26.0 $\pm$ 4.9 & 26.8 $\pm$ 4.7 & 1.06 $\pm$ 0.02 \\
         & & FA & 62.4 $\pm$ 3.1 & 62.4 $\pm$ 3.2 & 6.6 $\pm$ 2.0 & 6.1 $\pm$ 1.8 & 0.92 $\pm$ 0.02 \\
         & & FF (SGD) & 57.3 $\pm$ 2.3 & 57.3 $\pm$ 2.2 & 35.4 $\pm$ 19.6 & 41.4 $\pm$ 23.2 & 1.16 $\pm$ 0.02 \\
         & & FF (Adam) & 69.4 $\pm$ 0.3 & 69.4 $\pm$ 0.3 & 33.6 $\pm$ 5.1 & 44.3 $\pm$ 7.5 & 1.33 $\pm$ 0.01 \\
         & & DR & 61.1 $\pm$ 0.3 & 61.1 $\pm$ 0.3 & 10.8 $\pm$ 5.8 & 5.4 $\pm$ 2.9 & 0.50 $\pm$ 0.00 \\
         & & AP & 51.7 $\pm$ 5.4 & 51.7 $\pm$ 5.4 & 3.4 $\pm$ 6.8 & 1.0 $\pm$ 2.0 & 0.29 $\pm$ 0.00 \\
         & & CP & 51.7 $\pm$ 5.4 & 51.7 $\pm$ 5.4 & 3.4 $\pm$ 6.8 & 1.9 $\pm$ 3.9 & 0.57 $\pm$ 0.00 \\
        \cmidrule(lr){2-8}
         & \multirow{7}{*}{4 x 256 + Skip} & SGD & 68.9 $\pm$ 0.4 & 68.9 $\pm$ 0.4 & 73.0 $\pm$ 10.3 & 72.4 $\pm$ 10.9 & 0.99 $\pm$ 0.01 \\
         & & Adam & 76.2 $\pm$ 0.1 & 76.2 $\pm$ 0.1 & 22.6 $\pm$ 1.7 & 25.2 $\pm$ 1.9 & 1.15 $\pm$ 0.02 \\
         & & DR & 62.9 $\pm$ 0.6 & 62.9 $\pm$ 0.5 & 10.0 $\pm$ 3.6 & 6.4 $\pm$ 2.3 & 0.64 $\pm$ 0.00 \\
         & & AP & 58.2 $\pm$ 5.7 & 58.3 $\pm$ 5.8 & 13.0 $\pm$ 10.4 & 3.9 $\pm$ 3.1 & 0.30 $\pm$ 0.00 \\
         & & CP & 58.4 $\pm$ 5.8 & 58.4 $\pm$ 5.8 & 14.0 $\pm$ 10.4 & 6.7 $\pm$ 5.0 & 0.48 $\pm$ 0.00 \\
        \bottomrule
    \end{tabular}
    } 
\end{table}

\begin{table}[htbp]
    \centering
    \caption{CNN Performance Comparison.}
    \label{tab:cnn_results}
    \resizebox{\textwidth}{!}{ 
    \begin{tabular}{ @{} lllccccc @{} }
        \toprule
        Dataset & Architecture & Method & Test Acc. (\%) & Train Acc. (\%) & Steps (k) & Time (s) & Time/Step (ms) \\
        \midrule
        \multirow{19}{*}{MNIST} & \multirow{7}{*}{1 x 32} & SGD & 71.7 $\pm$ 1.1 & 70.8 $\pm$ 1.5 & 133.0 $\pm$ 23.9 & 402.8 $\pm$ 72.6 & 3.03 $\pm$ 0.00 \\
         & & Adam & 77.8 $\pm$ 1.2 & 77.1 $\pm$ 1.4 & 26.2 $\pm$ 3.4 & 81.7 $\pm$ 11.1 & 3.11 $\pm$ 0.02 \\
         & & FF (SGD) & 24.5 $\pm$ 4.0 & 24.1 $\pm$ 4.0 & 4.0 $\pm$ 2.8 & 13.3 $\pm$ 9.1 & 3.32 $\pm$ 0.04 \\
         & & FF (Adam) & 70.6 $\pm$ 3.4 & 69.9 $\pm$ 3.1 & 19.6 $\pm$ 6.0 & 65.1 $\pm$ 18.1 & 3.35 $\pm$ 0.11 \\
         & & DR & 51.7 $\pm$ 1.2 & 51.3 $\pm$ 1.5 & 12.4 $\pm$ 5.0 & 31.7 $\pm$ 12.8 & 2.55 $\pm$ 0.00 \\
         & & AP & 49.8 $\pm$ 1.6 & 49.4 $\pm$ 1.8 & 25.8 $\pm$ 2.3 & 72.2 $\pm$ 6.5 & 2.80 $\pm$ 0.00 \\
         & & CP & 49.9 $\pm$ 1.6 & 49.4 $\pm$ 1.8 & 25.6 $\pm$ 2.0 & 72.1 $\pm$ 5.6 & 2.82 $\pm$ 0.00 \\
        \cmidrule(lr){2-8}
         & \multirow{7}{*}{4 x 16} & SGD & 96.0 $\pm$ 0.2 & 96.3 $\pm$ 0.2 & 50.4 $\pm$ 8.0 & 161.0 $\pm$ 25.6 & 3.20 $\pm$ 0.03 \\
         & & Adam & 97.4 $\pm$ 0.3 & 99.0 $\pm$ 0.4 & 5.4 $\pm$ 1.0 & 18.4 $\pm$ 3.5 & 3.41 $\pm$ 0.02 \\
         & & FF (SGD) & 20.0 $\pm$ 3.0 & 19.5 $\pm$ 2.8 & 2.2 $\pm$ 1.9 & 8.1 $\pm$ 7.3 & 3.58 $\pm$ 0.10 \\
         & & FF (Adam) & 88.3 $\pm$ 1.2 & 87.7 $\pm$ 1.1 & 29.0 $\pm$ 6.0 & 108.3 $\pm$ 22.1 & 3.73 $\pm$ 0.03 \\
         & & DR & 55.4 $\pm$ 4.1 & 55.1 $\pm$ 3.8 & 99.4 $\pm$ 51.0 & 7713.7 $\pm$ 3959.3 & 77.60 $\pm$ 0.01 \\
         & & AP & 46.4 $\pm$ 4.8 & 46.0 $\pm$ 4.5 & 20.8 $\pm$ 3.3 & 1358.7 $\pm$ 216.8 & 65.32 $\pm$ 0.04 \\
         & & CP & 46.4 $\pm$ 4.8 & 46.1 $\pm$ 4.5 & 20.6 $\pm$ 3.0 & 1426.7 $\pm$ 208.6 & 69.24 $\pm$ 0.09 \\
        \cmidrule(lr){2-8}
         & \multirow{5}{*}{4 x 16 + Skip} & SGD & 96.7 $\pm$ 0.3 & 97.1 $\pm$ 0.3 & 51.4 $\pm$ 10.2 & 166.7 $\pm$ 33.4 & 3.25 $\pm$ 0.01 \\
         & & Adam & 97.9 $\pm$ 0.2 & 99.1 $\pm$ 0.3 & 4.0 $\pm$ 0.6 & 13.7 $\pm$ 2.2 & 3.40 $\pm$ 0.03 \\
         & & DR & 71.3 $\pm$ 1.6 & 70.8 $\pm$ 1.4 & 16.0 $\pm$ 6.4 & 1303.8 $\pm$ 518.2 & 81.47 $\pm$ 0.03 \\
         & & AP & 69.4 $\pm$ 1.4 & 68.8 $\pm$ 1.4 & 46.2 $\pm$ 15.2 & 3240.5 $\pm$ 1068.3 & 70.14 $\pm$ 0.02 \\
         & & CP & 69.6 $\pm$ 1.3 & 68.9 $\pm$ 1.2 & 44.2 $\pm$ 10.5 & 1609.5 $\pm$ 383.7 & 36.42 $\pm$ 0.03 \\
        \midrule
        \multirow{19}{*}{CIFAR-10} & \multirow{7}{*}{1 x 32} & SGD & 39.6 $\pm$ 1.1 & 40.0 $\pm$ 1.4 & 91.0 $\pm$ 19.2 & 821.3 $\pm$ 172.9 & 9.03 $\pm$ 0.03 \\
         & & Adam & 45.4 $\pm$ 0.6 & 45.9 $\pm$ 0.6 & 22.8 $\pm$ 6.3 & 210.7 $\pm$ 57.8 & 9.23 $\pm$ 0.03 \\
         & & FF (SGD) & 10.1 $\pm$ 0.0 & 10.0 $\pm$ 0.0 & 1.0 $\pm$ 0.0 & 9.7 $\pm$ 0.4 & 9.48 $\pm$ 0.13 \\
         & & FF (Adam) & 26.9 $\pm$ 2.1 & 27.0 $\pm$ 2.5 & 8.8 $\pm$ 4.0 & 84.3 $\pm$ 38.2 & 9.56 $\pm$ 0.05 \\
         & & DR & 27.6 $\pm$ 0.6 & 27.5 $\pm$ 0.9 & 2.4 $\pm$ 1.0 & 27.4 $\pm$ 11.6 & 11.41 $\pm$ 0.00 \\
         & & AP & 29.1 $\pm$ 0.9 & 29.1 $\pm$ 0.8 & 19.0 $\pm$ 5.4 & 110.6 $\pm$ 31.5 & 5.82 $\pm$ 0.00 \\
         & & CP & 29.2 $\pm$ 0.9 & 29.2 $\pm$ 0.9 & 18.8 $\pm$ 5.6 & 109.8 $\pm$ 32.7 & 5.84 $\pm$ 0.00 \\
        \cmidrule(lr){2-8}
         & \multirow{7}{*}{4 x 16} & SGD & 49.2 $\pm$ 1.8 & 51.2 $\pm$ 2.3 & 83.6 $\pm$ 25.9 & 786.5 $\pm$ 245.4 & 9.40 $\pm$ 0.05 \\
         & & Adam & 54.9 $\pm$ 0.8 & 62.8 $\pm$ 1.8 & 12.0 $\pm$ 4.0 & 114.2 $\pm$ 38.8 & 9.49 $\pm$ 0.09 \\
         & & FF (SGD) & 10.0 $\pm$ 0.1 & 10.1 $\pm$ 0.1 & 0.6 $\pm$ 0.5 & 6.0 $\pm$ 4.9 & 9.96 $\pm$ 0.33 \\
         & & FF (Adam) & 28.3 $\pm$ 3.3 & 28.4 $\pm$ 3.2 & 10.4 $\pm$ 6.3 & 102.3 $\pm$ 61.7 & 9.81 $\pm$ 0.09 \\
         & & DR & 20.3 $\pm$ 0.9 & 20.2 $\pm$ 0.6 & 2.8 $\pm$ 2.2 & 301.7 $\pm$ 240.1 & 107.78 $\pm$ 0.07 \\
         & & AP & 22.8 $\pm$ 1.4 & 22.5 $\pm$ 1.2 & 9.2 $\pm$ 5.1 & 827.1 $\pm$ 457.4 & 89.88 $\pm$ 0.17 \\
         & & CP & 22.6 $\pm$ 1.3 & 22.4 $\pm$ 1.2 & 9.0 $\pm$ 4.8 & 847.5 $\pm$ 448.7 & 94.44 $\pm$ 0.14 \\
        \cmidrule(lr){2-8}
         & \multirow{5}{*}{4 x 16 + Skip} & SGD & 54.1 $\pm$ 1.8 & 57.0 $\pm$ 2.1 & 106.8 $\pm$ 22.4 & 1004.6 $\pm$ 223.9 & 9.38 $\pm$ 0.16 \\
         & & Adam & 59.4 $\pm$ 0.6 & 67.6 $\pm$ 0.8 & 13.0 $\pm$ 3.2 & 123.9 $\pm$ 32.1 & 9.49 $\pm$ 0.08 \\
         & & DR & 31.4 $\pm$ 1.2 & 31.8 $\pm$ 1.6 & 3.2 $\pm$ 0.4 & 365.3 $\pm$ 45.5 & 114.13 $\pm$ 0.18 \\
         & & AP & 33.4 $\pm$ 1.0 & 33.5 $\pm$ 1.1 & 24.6 $\pm$ 4.6 & 2352.5 $\pm$ 439.3 & 95.68 $\pm$ 0.06 \\
         & & CP & 33.7 $\pm$ 1.2 & 33.8 $\pm$ 1.2 & 24.8 $\pm$ 4.6 & 1433.4 $\pm$ 266.9 & 57.82 $\pm$ 0.02 \\
        \bottomrule
    \end{tabular}
    } 
\end{table}

\begin{table}[htbp]
    \centering
    \caption{RNN Performance Comparison.}
    \label{tab:rnn_results}
    \resizebox{\textwidth}{!}{ 
    \begin{tabular}{ @{} lllccccc @{} }
        \toprule
        Dataset & Architecture & Method & Test Acc. (\%) & Train Acc. (\%) & Steps (k) & Time (s) & Time/Step (ms) \\
        \midrule
        \multirow{12}{*}{Shakespeare} & \multirow{4}{*}{1 x 256} & SGD & 35.7 $\pm$ 5.1 & 38.8 $\pm$ 6.4 & 179.0 $\pm$ 81.4 & 480.9 $\pm$ 218.7 & 2.69 $\pm$ 0.00 \\
         & & Adam & 49.9 $\pm$ 0.0 & 59.8 $\pm$ 0.4 & 7.4 $\pm$ 1.4 & 20.0 $\pm$ 3.7 & 2.70 $\pm$ 0.00 \\
         & & DR & 32.4 $\pm$ 0.5 & 34.0 $\pm$ 0.7 & 16.6 $\pm$ 5.5 & 849.0 $\pm$ 283.4 & 51.14 $\pm$ 0.02 \\
         & & AP & 29.6 $\pm$ 0.8 & 30.4 $\pm$ 1.0 & 98.8 $\pm$ 56.2 & 1646.0 $\pm$ 935.6 & 16.66 $\pm$ 0.00 \\
        \cmidrule(lr){2-8}
         & \multirow{4}{*}{4 x 128} & SGD & 35.1 $\pm$ 4.1 & 37.1 $\pm$ 5.3 & 63.0 $\pm$ 34.7 & 255.3 $\pm$ 140.8 & 4.05 $\pm$ 0.00 \\
         & & Adam & 49.8 $\pm$ 0.1 & 57.9 $\pm$ 0.3 & 28.6 $\pm$ 8.3 & 118.4 $\pm$ 34.3 & 4.14 $\pm$ 0.00 \\
         & & DR & 26.7 $\pm$ 1.1 & 27.4 $\pm$ 1.2 & 10.4 $\pm$ 6.7 & 234.1 $\pm$ 149.6 & 22.51 $\pm$ 0.00 \\
         & & AP & 26.5 $\pm$ 1.1 & 27.1 $\pm$ 1.2 & 182.4 $\pm$ 84.0 & 1688.6 $\pm$ 777.3 & 9.26 $\pm$ 0.00 \\
        \cmidrule(lr){2-8}
         & \multirow{4}{*}{4 x 128 + Skip} & SGD & 38.9 $\pm$ 4.5 & 42.6 $\pm$ 6.1 & 129.6 $\pm$ 80.2 & 619.7 $\pm$ 383.6 & 4.78 $\pm$ 0.00 \\
         & & Adam & 50.2 $\pm$ 0.3 & 59.0 $\pm$ 0.7 & 15.8 $\pm$ 5.2 & 75.1 $\pm$ 24.7 & 4.75 $\pm$ 0.00 \\
         & & DR & 32.2 $\pm$ 0.8 & 33.9 $\pm$ 0.9 & 32.8 $\pm$ 8.4 & 1210.8 $\pm$ 308.3 & 36.91 $\pm$ 0.00 \\
         & & AP & 30.2 $\pm$ 0.5 & 31.1 $\pm$ 0.7 & 205.2 $\pm$ 78.3 & 2964.4 $\pm$ 1131.2 & 14.44 $\pm$ 0.01 \\
        \bottomrule
    \end{tabular}
    } 
\end{table}

\end{document}